\documentclass{article}

\usepackage{arxiv}

\usepackage[utf8]{inputenc} 
\usepackage[T1]{fontenc}    
\usepackage{hyperref}       
\usepackage{url}            
\usepackage{booktabs}       
\usepackage{amsfonts}       
\usepackage{nicefrac}       
\usepackage{microtype}      
\usepackage{lipsum}		
\usepackage{graphicx}
\usepackage{natbib}
\usepackage{doi}

\usepackage{amsmath}
\usepackage{amsthm}
\usepackage{amssymb}
\usepackage{enumerate}
\usepackage{xcolor}
\usepackage{array}   
\usepackage{url} 
\usepackage{tikz} 
\usetikzlibrary {positioning}
\usepackage{subcaption}

\usepackage{pgfplotstable}
\usepackage{pgfplots}

\usepgfplotslibrary{fillbetween}
\usepgfplotslibrary[colorbrewer]
\pgfplotsset{compat=1.18}
\pgfplotsset{
    cycle list/Paired-4,
    cycle multiindex* list={
        mark list*\nextlist
        Paired-4\nextlist
    },
} 

\title{Impact of Fairness Regulations on Institutions' Policies and Population Qualifications}

\date{}
\author{
Hamidreza Montaseri${}^{1}$ \hspace{1cm} Amin Gohari${}^{2}$ \AND \\
${}^1$Tehran Institute of Advanced Studies, \texttt{montaseri.hamidreza@gmail.com} 
\\
${}^2$The Chinese University of Hong Kong, \texttt{agohari@ie.cuhk.edu.hk} 
}

\renewcommand{\shorttitle}{Fairness Regulation Impact}

\hypersetup{
pdftitle={Fairness Regulation Impact},
pdfsubject={q-bio.NC, q-bio.QM},
pdfauthor={Hamidreza Montaseri},
pdfkeywords={AI, Fairness},
}

\newcommand{\Pqt}[3]{p_{#1}(#2|#3)}
\newcommand{\Sct}[1]{\Pr(D_t=1|#1)}
\newcommand{\Sctinv}[1]{\Pr(D_t=0|#1)}
\newcommand{\Sc}[2][]{\Pr^{#1}(D=1|#2)}

\newcommand{\F}{\Delta}
\newcommand{\U}{u}
\usepackage{BOONDOX-uprscr}
\newcommand{\A}{\mathscr A}
\newcommand{\B}{\mathscr B}

\newcommand{\Uum}{u_{\text{UM}}}
\newcommand{\G}{\mathscr{C}}
\newcommand{\E}{\mathbf{E}}
\newcommand{\R}{\mathbb{R}}
\newcommand{\Delnp}{\Delta_{\text{UM}}}
\DeclareMathOperator*{\argmin}{argmin}
\DeclareMathOperator*{\argmax}{argmax}
\DeclareMathOperator*{\maximize}{maximize}

\newtheorem{definition}{Definition}
\newtheorem{theorem}{Theorem}
\newtheorem{remark}{Remark}
\newtheorem{example}{Example}

\newtheorem{lemma}{Lemma}

\begin{document}

\maketitle

\begin{abstract}
The proliferation of algorithmic systems has fueled discussions surrounding the regulation and control of their social impact. Herein, we consider a system whose primary objective is to maximize utility by selecting the most qualified individuals. To promote demographic parity in the selection algorithm, we consider penalizing discrimination across social groups. We examine conditions under which a discrimination penalty can effectively reduce disparity in the selection. Additionally, we explore the implications of such a penalty when individual qualifications may evolve over time in response to the imposed penalizing policy. We identify scenarios where the penalty could hinder the natural attainment of equity within the population. Moreover, we propose certain conditions that can counteract this undesirable outcome, thus ensuring fairness.
\end{abstract}

\section{Introduction}

The application of data-driven and algorithmic decision-making has rapidly accelerated over the past few decades. Many of these applications are of significant social importance and profoundly impact individuals. One particular concern regarding the deployment of algorithmic decision-making systems is the potential influence of societal biases on their performance, as well as their reciprocal impact on society. Numerous studies have extensively examined various aspects of this problem. Furthermore, societies are increasingly advocating for regulations to address and mitigate biases within these systems and, in some cases, even leverage them to reduce biases prevalent within society itself (\cite{2021_ai_act, smuha_race_to_regulation}).

One common concept of fairness in these discussions is demographic parity. For a system that distributes a given resource (such as loans, college admissions, or employment opportunities) within a population, it requires to ensure that each protected group within the population (e.g., gender or racial groups) receives a proportionate share relative to its size (\cite{barocas-hardt-narayanan}). In this study, we analyze the impact of enforcing this notion on utility-maximizing institutions, focusing on their policies and their effect on the population structure. Previous studies have commonly approached this problem by considering scenarios where policies must fully comply with demographic parity. However, our investigation considers a regulatory framework that imposes a penalty based on the disparity in institutions' policy, which aims to select individuals with higher qualifications. This introduces a trade-off between their utility and the associated penalty.

First, in a single-shot, static setting, we identify cases where the imposition of the penalty does not affect institutional policies. We propose a formulation for determining the requisite penalty level
that would compel the institution to deviate from its default policies in the absence of any penalty. The shape of the penalty function is shown to lead to different qualitative behaviors on the requisite penalty level. This is illustrated by comparing linear or quadratic penalty functions.

We also discuss when the penalty ensures complete demographic parity and how, in some cases, certain forms of penalty cannot guarantee this outcome.
Next, we explore a setting where institutional policies can influence the qualifications of individuals over time, thereby affecting the composition of the population. Similar to previous works (but in a different setting), we find instances in our model where fairness regulations can prevent the natural convergence of equal qualification distribution across groups. Despite that, we identify conditions that can prevent this undesirable effect.

Finally, we empirically examine our theoretical results using the Law School Admission Council's dataset (\cite{Wightman1998LSACNL}). The experiments are consistent with the theoretical results both in the static and dynamic cases (where we use synthesized dynamics). The experiments confirm that our framework corresponds to real-world scenarios and can be utilized in the regulation of discrimination and fairness.

\section{Problem Formulation in the Static Case}
First, we investigate the impact of fairness regulation in a static setting. We consider a scenario in which each member of a population possesses a feature vector denoted as $x \in \mathcal{X}$, which is accessible to an institution that makes selections over them. Furthermore, individuals can be classified into distinct groups, such as racial or gender groups, and achieving fairness across these groups is a primary concern. For simplicity of exposition, we assume that there are two groups $\A$ and $\B$, and each individual belongs to either group $\A$ or $\B$. This membership is represented by the variable $c \in \G = \{\A, \B\}$.
The qualification sought by the institution for an individual is represented by the variable $y \in \mathcal{Y} \subset \R$. Having $y>0$ means that the institution will have a positive net benefit from selecting the individual, and having $y<0$ means that the costs of selecting the individual outweigh the benefits. Throughout, we assume $\mathcal{Y}$ as a finite set and $y\neq 0$, i.e., the selection of each individual is either beneficial or harmful. We assume an underlying joint distribution $p(c,x,y)$ for an individual selected uniformly at random from the society.

The institution employs a policy to determine whether to select an individual based on her feature $x$ and group membership $c$. This decision is represented by the binary variable $D \in \{0,1\}$. The institution's selection policy can be modeled by a conditional distribution $p(d|c,x)$. Thus, the joint distributions of the random variables $D, C, X, Y$ factorizes as $$p(d,c,x,y)=p(d|c,x)p(c,x,y).$$

The institution's utility is defined as the expected value of the qualification of the selected individuals, denoted as $\U = \E[YD]$. It is important to note that this utility can be interpreted as the expected utility for a randomly selected individual who undergoes the institution's policy or as the average utility when the policy is applied to the entire population. The randomness here arises from the distribution of features and qualifications in the population, as well as any random elements present in the selection policy itself.

To enforce demographic parity, we introduce a discrimination penalty based on selection disparity across groups denoted by $\F$:
\begin{align}\label{defDelta}
\F = \left| \Pr(D=1|C=\A) - \Pr(D=1|C=\B)\right|.
\end{align}
Consequently, the objective of the institution in selecting its policy can be expressed as follows:
\begin{gather*}
\maximize_{p(d|c,x)}\; \E[DY] - \lambda \cdot g(\F)
\label{eq:base}
\end{gather*}
where $g(\cdot)$ is a non-decreasing and non-negative convex function satisfying $g(0)=0$,
and $\lambda\geq 0$. Therefore, $\lambda \cdot g(\Delta)$ determines the penalty imposed on the institution (discrimination penalty). Discrimination penalty $\lambda \cdot g(\F)$ can represent the fine imposed by the government on the institution to enforce demographic parity. It may also represent the negative effects of discrimination on the company, such as reduced productivity, decreased commitment, and employee demotivation, resulting in an overall cost to the institution. Thus, the true objective of the company can be written as $\mathsf{Profit}=\E[DY] - \lambda \cdot g(\F)$. As an example, consider $$g(\Delta)=\begin{cases}
0,&0\leq \Delta\leq \Delta^*\\
\Delta-\Delta^*,&\Delta> \Delta^*
\end{cases}.$$
This cost function sets a level $\Delta^*$ of tolerable discrimination below which no penalty is imposed.

Throughout this paper, we assume that the institution knows the underlying joint distribution $p(c,x,y)$ and can exactly solve the above optimization problem. This assumption represents an idealization of real-world scenarios, where such problems are typically addressed using approximate solutions based on the available data. Nonetheless, the insights gained from this idealized model remain applicable and valuable in practical settings.

\subsection{A simplification}\label{secsimp}

We claim that without loss of generality, we may assume that the institution has access to the qualification information of individuals, meaning that $Y$ is a function of $C$ and $X$. To see this, given any $x,c$, we define
$$\hat Y=\E[Y|X,C].$$
Thus, $\hat Y$ is a function of $(X,C)$. Then, we have
\begin{align*}
\E[DY] - \lambda \cdot g(\F)=\E\left[D\E[Y|X,C]\right]- \lambda \cdot g(\F)=\E[D\hat Y] - \lambda \cdot g(\F).
\end{align*}
Therefore, we obtain the same utility for $\hat Y$ in place of $Y$. This shows that we can restrict ourselves to the case of $Y$ being a function of $C$ and $X$.

Next, observe that the objective function $\E[DY] - \lambda \cdot g(\F)$ depends only on the marginal distribution on $(D,C,Y)$ (rather than the full joint distribution of $(D,C,Y,X)$). Consequently, we can write the optimization problem as a maximum over $p(d|c,y)$ as follows:
\begin{gather}
\maximize_{p(d|c,y)}\; \E[DY] - \lambda \cdot g(\F).
\label{eq:prob}
\end{gather}

\subsection{Related Work}
Since the start of the discourse on AI fairness, a line of work has focused on how to apply different notions of fairness in algorithmic decision-making. For instance, in machine learning, numerous efforts have been made to investigate the efficiency and performance of various approaches for aligning learning models with fairness notions (\cite{fair_ml_survey}).
Additionally, an important observation is that some common notions are inherently conflicting (\cite{kleinberg_inherent_2016, chouldechova_fair_2017}); therefore, a decision-maker cannot simultaneously satisfy them. In this paper, we abstract the details of the decision-making procedure and adopt a specific definition of fairness (demographic parity) to focus on this issue from a regulator's perspective. When and how can fairness be effectively enforced through penalization? To study this question,  we have constructed a model that extends the models examined in the existing literature. Specifically, we employ a variant of the models proposed in previous works such as \cite{mozannar_fair_2020,zhang_how_2020,raab_unintended_2021} which consider policies constrained to satisfy demographic parity, i.e., $\Delta=0$. However, in our work, we formulate fairness regulation as a penalty imposed by a regulator. If we impose an infinite discrimination penalty by letting $\lambda$ converge to infinity in \eqref{eq:prob}, we recover the requirement $\Delta=0$ as a special case. Moreover, the models proposed by \cite{mozannar_fair_2020,zhang_how_2020,raab_unintended_2021} take the support set
$\mathcal{Y}$ to be a binary set, while we take it to be an arbitrary finite set. Finally, we note that \cite{mozannar_fair_2020} assumes that the institution makes decisions based on the group identity $C$ and qualification $Y$, while \cite{zhang_how_2020} and \cite{raab_unintended_2021} consider the institution's policy as the optimal solution of an optimization problem similar to the one considered by us. We showed that these two assumptions are equivalent in our model.

We have dedicated Section \ref{sectiondynamics} to exploring the dynamics of qualifications in society over time. The interaction between the dynamics of society and algorithmic systems has recently garnered attention from the responsible AI community. As one of the first efforts in this direction, \cite{liu_delayed_2018} considered a two-stage model. In the first stage, individuals within a population receive a decision based on some qualification (e.g., credit score). Then, in the next stage, this decision influences that same qualification. The authors showed that in this model, measures like demographic parity and equality of opportunity can have a negative impact and exacerbate disparities between groups in the population. We also observe a similar phenomenon in our model.

There exist other approaches to fairness that we leave for future work. In particular, one can also analyze the interplay between society and algorithmic systems by employing game-theoretic formulations. Authors in \cite{liu_disparate_2019} consider both sides as strategic agents, wherein the system benefits from correct classification, and individuals benefit from selection and can invest in increasing their chances of being selected. The authors investigate the equilibrium of this setting and relate it to the realizability of classification. In a different framework, \cite{2016_hardt_strategic} considers a strategic classification framework and learnability in the presence of strategic individuals. Shifting the focus to the societal perspective, \cite{2019_milli_social_cost} extends this framework and shows that when the system takes into account the population's strategic behavior, it imposes a greater burden on disadvantaged groups.

\section{Effectiveness of the Discrimination Penalty in the Static Case}
\label{sec:effstat}
If there is no penalty on the amount of discrimination, i.e., $\lambda=0$, the institution will maximize $\E[DY]$. Let $Y_+$ be the positive part of $Y$, i.e., $Y_+=\max(0, Y)$. Then, since $D\in \{0,1\}$, we have
$\E[DY]\leq \E[Y_+]$ and equality is achieved by a unique policy of the institution selecting an individual when $Y > 0$.
We define \begin{equation}\Uum=\E[Y_+]\label{uumdef}\end{equation} to be the utility when $\lambda=0$.
The resulting discrimination, denoted by $\Delnp$, simplifies as follows:
\begin{align}
   \Delnp&=
    |\Sc{ C=\A} - \Sc{C=\B}| \nonumber
\\&=
   |\Pr(Y > 0| C=\A) - \Pr(Y > 0| C=\B)|.\label{defDNp}
    \end{align}
Thus, $\Delnp$, the discrimination for the optimal policy when $\lambda=0$, can be computed explicitly.

Note that if $\Pr(Y>0|C=\A) =\Pr(Y>0|C=\B)$, we have $\Delnp=0$ and there is no tension between maximizing $\E[DY]$ and minimizing $\Delta$. Thus, the more interesting case is when $\Delnp>0$, \emph{i.e.,} $\Pr(Y>0|C=\A) \neq \Pr(Y>0|C=\B)$. Throughout the paper, when discussing the static case, we assume (without loss of generality) that individuals in group $\A$ are relatively more beneficial to the institutions than individuals in group $\B$, i.e.
\begin{align}
    \Pr(Y>0|C=\A) > \Pr(Y>0|C=\B)\label{assumptione}
\end{align}
If this is not the case, we can swap the names for groups $\A$ and $\B$.

Unlike $\lambda=0$, the optimal policy is not necessarily unique when $\lambda>0$. This property and some other elementary properties of optimal policies of the institution are discussed in Lemma \ref{lem1} in Appendix \ref{app:proof}. One important property to highlight here is that for any $\lambda$ and $g(\cdot)$, under any optimal policy, the advantaged group never has a smaller selection rate, i.e.,
\[\Sc{C=\A} \geq \Sc{C=\B}. \]
Consequently,  under any optimal policy, \eqref{defDelta} can be rewritten as
\begin{align}\label{defDelta2}
\F = \Pr(D=1|C=\A) - \Pr(D=1|C=\B).
\end{align}

We begin by considering when the penalty is effective meaning it will force the institution to abandon the utility-maximizing policy and reduce disparity in selection across groups.

\begin{definition}As above, assume $\Delnp>0$.
    We call a penalty, given by function $g(\cdot)$ and scalar $\lambda$, effective when every optimal solution of the problem in \eqref{eq:prob} satisfies
    \begin{align*}
        \Delta=
    |\Sc{ C=\A} - \Sc{C=\B}| <\Delnp= \Pr(Y > 0| C=\A) - \Pr(Y > 0| C=\B).
    \end{align*}
\end{definition}

The next theorem states the condition to ensure the effectiveness of the penalty which is dependent on the size of the groups, possible qualifications $\mathcal{Y}$, and the bias present in the population.

\begin{theorem}
    \label{th:effective}
   The penalty function $g(\cdot)$ and scalar $\lambda$ are effective if and only if
    $\beta_e < \lambda \cdot g'_-(\Delnp)$,
    where $g'_-(\cdot)$ is the  left-hand derivative of $g$
$$g'_-(x)=\lim_{h \mathop \to 0^-} \frac {g(x + h) - g(x)} h$$
    and
$$\beta_e = \min_{(y,c)\in\mathcal{T}_e} p(c)\cdot|y|$$
where $\mathcal{T}_e\subset\mathcal{Y}\times\{\A,\B\}$ is defined as
\begin{align*}
\mathcal{T}_e=\{(y,\A): y>0, p(y|C=\A)>0\}\cup \{(y,\B): y<0, p(y|C=\B)>0\}.
\end{align*}
\end{theorem}
Note that \eqref{assumptione} guarantees that the set $\mathcal{T}_e$ is non-empty. Proof of the above theorem can be found in Appendix \ref{appndxproofth:effective}.

\begin{remark}
 Observe that
$\beta_e$ depends on the distribution of qualifications across groups $p(y|c)$ only through its support set, i.e., only the set of $y$ where $p(y|c)>0$ is important and not the exact values of $p(y|c)$.
Next, note that effectiveness only depends on the derivative of the discrimination penalty at $\Delnp$. The choice of $g(\cdot)$ is also important. For instance, consider two functions $g_1(x)=x$ and $g_2(x)=x^2$. As $g_1(x)$ has a constant derivative, the threshold of $\lambda$ for effectiveness is found by $\beta_e < \lambda$ which is independent of $\Delnp$. But in the case of the square function $g_2(x)$, the threshold will be $\frac{\beta_e}{2\Delnp} < \lambda$,  so the current level of bias in the population affects the effectiveness of the penalty. Therefore assessing the effectiveness of the penalty becomes more complex. Also, in an evolving population (as in the next section) the effectiveness may change over time.
\end{remark}

Next, we discuss the possibility of reaching complete demographic parity, $\Sc{ C=\A} = \Sc{C=\B}$. Note that the institution always has the option of adopting a policy that satisfies complete demographic parity, i.e.,
\begin{align}
\max_{p(d|c,y)}\; \E[DY] - \lambda \cdot g(\F)\geq &\max_{
\substack{p(d|c,y):  \\ p(D=1|C=\A) = p(D=1|C=\B)}
}\; \E[DY].
\end{align}
We are interested in knowing when the discrimination penalty is strong enough to ensure complete demographic parity in the institution's policy. This concept is formalized in the definition below:
\begin{definition}
    We call a penalty, given by function $g(\cdot)$ and scalar $\lambda$, fully satisfactory when at least one optimal solution of the problem in \eqref{eq:prob} satisfies
    \[\Delta=
    |\Sc{ C=\A} - \Sc{C=\B}|=0.
    \]
    Equivalently,
    \begin{align}
\max_{p(d|c,y)}\; \E[DY] - \lambda \cdot g(\F)=&\max_{
\substack{p(d|c,y):  \\ p(D=1|C=\A) = p(D=1|C=\B)}
}\; \E[DY].
\end{align}
\end{definition}
We have the following result:
\begin{theorem}
\label{th:fulldp}
For a given distribution $p(y,c)$, if $\Delnp > 0$ then, the penalty function $g(\cdot)$ and scalar $\lambda$ are fully satisfactory if and only if
    $\beta_s \leq \lambda \cdot g'_+(0) $,
where
$g'_+(\cdot)$ is the  right-hand derivative of $g$:
$$g'_+(x)=\lim_{h \mathop \to 0^+} \frac {g(x + h) - g(x)} h$$
$\beta_s$ is defined as follows:
Consider the problem of finding the best policy that satisfies demographic parity:
\begin{align} \label{eq:demo-parity}
\max_{
\substack{p(d|c,y):  \\ p(D=1|C=\A) = p(D=1|C=\B)}
}\; \E[DY].
\end{align}
Let $p^*(d|c,y)$ be any optimal solution to the above optimization (if there is more than one optimal policy, we can choose any one of them). Then, let
$\mathcal{T}_s\subset\mathcal{Y}\times\{\A,\B\}$ be defined as
\begin{align*}
\mathcal{T}_s=&\{(y,\A): y>0, p(y|\A)>0,p^*(D=1|y,\A) < 1\}
\cup \{(y,\B): y<0, p(y|\B)>0,p^*(D=1|y,\B) > 0\}.
\end{align*}
Finally, set
$$\beta_s = \max_{(y,c)\in\mathcal{T}_s} p(c)\cdot|y|.$$
It turns out that the value of $\beta_s$ will be the same if we start from different optimal solutions of \eqref{eq:demo-parity}.
\end{theorem}
Proof of the above theorem can be found in Appendix \ref{appndxproofth:fulldp}.

In this case, the condition depends on the policy with demographic parity constraint. Consequently, employing this policy to assess a discrimination penalty is not as straightforward as in the previous theorem. Nevertheless, it underscores important implications in selecting the penalty function $g(\cdot)$. Following our comparison of the functions $g_1(x) = x$ and $g_2(x) = x^2$ in the context of Theorem \ref{th:effective}, note that for the square function, we have $g'_+(0) = 0$. Consequently, for any finite value of $\lambda$, the condition $\beta_s > \lambda \cdot 0$ is satisfied, thereby making it impossible to guarantee complete demographic parity, except asymptotically as $\lambda$ tends to infinity. However, when considering $g_1(x)=x$, the condition for achieving complete demographic parity simplifies to $\beta_s < \lambda$, indicating the existence of some finite level of penalty sufficient to achieve demographic parity. Once again, this highlights a distinction between these two penalty functions.

\section{Dynamics of the Population} \label{sectiondynamics}
After discussing fairness regulations at a static, single point in time, we now shift our focus towards understanding how the population may dynamically evolve over time in response to the selection policy. Real-world instances of such transformations are varied, ranging from individuals investing in their attributes to obtain improved future outcomes, to situations where receiving a loan based on one's credit score can further enhance that score.

We assume that individuals' group membership does not change over time. However, qualifications are states that individuals can acquire and may change over time. Extending our notation from the previous section, we denote a randomly chosen individual's qualification at time $t$ by $Y_t$. We assume that the institution makes a decision about the individual at every instance $t$ and denote it by the random variable $D_t$.
The decision at time $t$ ($D_t$) has the potential to alter the qualifications of individuals in the next time step. We make the following Markov property assumption governing these changes:
\begin{align}\label{MCassumption}
p(y_{t+1} | y_1, \cdots, y_t, d_1, \cdots, d_t, c) = p(y_{t+1} | y_t, d_t)
\end{align}
We borrow the above assumption from \cite{mozannar_fair_2020,zhang_how_2020} and \cite{raab_unintended_2021}. Observe that this assumption implies that the new qualification $Y_{t+1}$ depends only on the current qualification $Y_t$ and the current decision by the institution $D_t$, but not on group $C$ or past history. Additionally, following the same literature, we assume that the conditional probability $\Pr(Y_{t+1} = y | Y_t=y', D_t)$ does not change over time, i.e.,
\begin{gather*}
    \Pr(Y_{t+1} = y| Y_t = y', D_t = d) = q(y|y',d)
\end{gather*}
for some conditional distribution $q(y|y',d)$ that does not depend on time $t$.

Note that the decision variable $D_t$ depends not only on the individual's qualification $Y_t$, but possibly also on the population \emph{distribution} of qualifications across the society, i.e., the joint distribution of $(C,Y_t)$. In other words, $p(d_t|c,y_t)$ may depend on the distribution $p(c,y_t)$ in order to ensure fairness across the society. Consequently, the qualification of a randomly selected individual over time $\{Y_t\}_{t=1,2,\cdots}$ does not form a Markov chain despite the Markov assumption in \eqref{MCassumption}.

In this section, we examine the institution's adoption of myopic policies, which maximize the instantaneous objective discussed in previous sections at every time step.
In other words, we assume that the institution solves the following optimization problem in each step $t$:
\begin{gather}
\maximize_{p(d_t|c,y_t)}\; \E[D_tY_t] - \lambda \cdot g(\F_t)\label{eqnformulationd}
\end{gather}
where
 \[\Delta_t=
    |\Sct{ C=\A} - \Sct{C=\B}|.
    \]
If there are multiple maximizers $p(d_t|c,y_t)$, the institution may choose any of them.
Note that, as in the previous section, the policy of the institution at time $t$ depends only on the marginal distribution of $C$ and $Y_t$.
So, the state of the system at each time step $t$ is determined by the distribution of qualifications of each group, denoted by $p_{Y_t|\A}$ and $p_{Y_t|\B}$, as well as the choice of maximizer if the optimization problem in \eqref{eqnformulationd} has more than one maximizer. We call a state stationary when it does not change in the next time step.

\begin{definition}\label{def3}
    For an instance of the system, given by group sizes $\Pr(C=\B)$ and $\Pr(C=\B)$, discrimination penalty $g(\cdot)$ and $\lambda$ and the conditional distribution $q(y| y',d)$, a state $(p_{Y|\A}, p_{Y|\B}, p_{D|C,Y})$ is stationary if starting from $p_{Y_t|\A}=p_{Y|\A}, p_{Y_t|\B}=p_{Y|\B}$, we have that
    $p_{D|C,Y}$ is a maximizer for the optimization problem in \eqref{eqnformulationd} and the state does not change in the next time step:
    \begin{align*}
       p_{Y_t|\A} = p_{Y_{t+1}|\A}, \\
       p_{Y_t|\B} = p_{Y_{t+1}|\B}.
    \end{align*}
\end{definition}

Stationary states are interesting because if the system reaches them at some point in time, it will not move to any other new state in the future.

Within the literature exploring societal dynamics and algorithmic decision-making, social equality between groups within a population is often characterized by the condition $\E[Y_t|C=\A]  = \E[Y_t|C=\B]$ (\cite{mozannar_fair_2020,zhang_how_2020,raab_unintended_2021}). This condition alone does not guarantee equal selection outcomes under a utility-maximizing policy (unless $Y_t$ is a binary random variable). More generally, in this paper, we study the
joint distribution of $(C, Y_t)$. The statements we prove can be used to study the difference between $ \E[Y_t|C=\A] $ and $ \E[Y_t|C=\B]$.

Our goal is to study the impact of $\lambda$ and penalty function $g(\cdot)$ on the limit of
\[\lim_{t\rightarrow\infty}\Delta_t=
\lim_{t\rightarrow\infty}|\Sct{ C=\A} - \Sct{C=\B}|.
\]

First, we show that under some mild assumptions, if $\lambda=0$ in \eqref{eqnformulationd}, the amount of discrimination $\Delta_t$ vanishes over time. In other words, even when we do not apply any discrimination penalty, both groups will end up with the same qualification distribution in the long run. This is because we use the same update rule $q(y|y',d)$ for qualifications across different groups (the new qualification $Y_{t+1}$ depends only on the current qualification $Y_t$ and the current decision by the institution $D_t$, \emph{but not on the individual's group identity $C$}). In other words, equal opportunities for growth or decline are provided to different groups. In this scenario, Theorem \ref{th:um} below informs us that even if there is initially a disparity in qualifications, which implies uneven selection under a policy that maximizes utility, equal opportunities for change will eventually result in equal distributions across the groups in the long run.

On the other hand, Theorem \ref{th:um} below shows that the amount of discrimination $\Delta_t$ may not vanish over time when $\lambda>0$. This is rather surprising because by choosing a positive $\lambda$, we penalize discrimination at each step. However, we observe that adopting such a myopic policy in some cases might result in non-vanishing discrimination in the long run, i.e., $\lim_{t\rightarrow\infty}\Delta_t>0$. Furthermore, it is possible that $\lim_{t\rightarrow\infty}\Delta_t>\Delta_0$, i.e., the limiting discrimination value may be higher than the original discrimination in the society at time $t=0$.

\begin{theorem}
\label{th:um}
Assume that $q(y|y',d)>0$ for all $y,y'\in\mathcal{Y}$ and $\mathcal{Y}$ is a finite set. Then, if $\lambda=0$, we have
\[\lim_{t\rightarrow\infty}\Delta_t=0.
    \]
    More generally, $C$ and $Y_t$ will become asymptotically independent:
    \[
\lim_{t\rightarrow\infty}\|p_{C,Y_t}-p_{C}p_{Y_t}\|_1=0
    \]
    where $\|\cdot\|_1$ is the total variation distance. However, one can find examples for $\lambda>0$ in which
    \[\lim_{t\rightarrow\infty}\Delta_t>\Delta_0>0.
    \]
\end{theorem}
In summary, while there may be a desire to penalize discrimination as a morally commendable action or to mitigate biases within the population, the above theorem suggests that this approach may, unfortunately, prevent achieving equality among the groups in the long run. We note that this negative impact is similar to what has been described (in different settings) by \cite{mozannar_fair_2020,zhang_how_2020} and \cite{raab_unintended_2021}, where the enforcement of certain fairness criteria in decision-making can prevent society from naturally attaining equality.

\begin{figure}
    \centering
    \begin{subfigure}{0.45\textwidth}
    \begin {tikzpicture}[-latex ,auto ,node distance =3 cm ,on grid ,
    semithick ,
    state/.style ={ circle ,draw}]
    \node[state] (A){$-1$};
    \node[state] (B) [above left=of A] {$-2$};
    \node[state] (C) [above right =of A] {$+2$};
    \path (A) edge [loop below] node[below] {$0.1$} (A);
    \path (B) edge [loop left] node[left] {$0.8$} (B);
    \path (C) edge [loop right] node[right] {$0.8$} (C);
    \path (A) edge [bend right =35] node[below =0.15 cm] {$0.1$} (C);
    \path (C) edge [bend left =-15] node[below =0.15 cm] {$0.1$} (A);
    \path (A) edge [bend left =35] node[below =0.15 cm] {$0.8$} (B);
    \path (B) edge [bend left =15] node[below =0.15 cm] {$0.1$} (A);
    \path (C) edge [bend left =15] node[below =0.15 cm] {$0.1$} (B);
    \path (B) edge [bend right = -15] node[above =0.15 cm] {$0.1$} (C);
    \end{tikzpicture}
        \caption{$D=1$}
    \end{subfigure}
    \begin{subfigure}{0.45\textwidth}
    \centering
    \begin {tikzpicture}[-latex ,auto ,node distance =3 cm ,on grid ,
    semithick ,
    state/.style ={ circle ,draw}]
    \node[state] (A){$-1$};
    \node[state] (B) [above left=of A] {$-2$};
    \node[state] (C) [above right =of A] {$+2$};
    \path (A) edge [loop below] node[below] {$0.1$} (A);
    \path (B) edge [loop left] node[left] {$0.8$} (B);
    \path (C) edge [loop right] node[right] {$0.8$} (C);
    \path (A) edge [bend right =35] node[below =0.15 cm] {$0.8$} (C);
    \path (C) edge [bend left =-15] node[below =0.15 cm] {$0.1$} (A);
    \path (A) edge [bend left =35] node[below =0.15 cm] {$0.1$} (B);
    \path (B) edge [bend left =15] node[below =0.15 cm] {$0.1$} (A);
    \path (C) edge [bend left =15] node[below =0.15 cm] {$0.1$} (B);
    \path (B) edge [bend right = -15] node[above =0.15 cm] {$0.1$} (C);
    \end{tikzpicture}
        \caption{$D=0$}
    \end{subfigure}
    \caption{Transition probabilities when a) getting selected and b) when not}
    \label{fig:tran}
\end{figure}
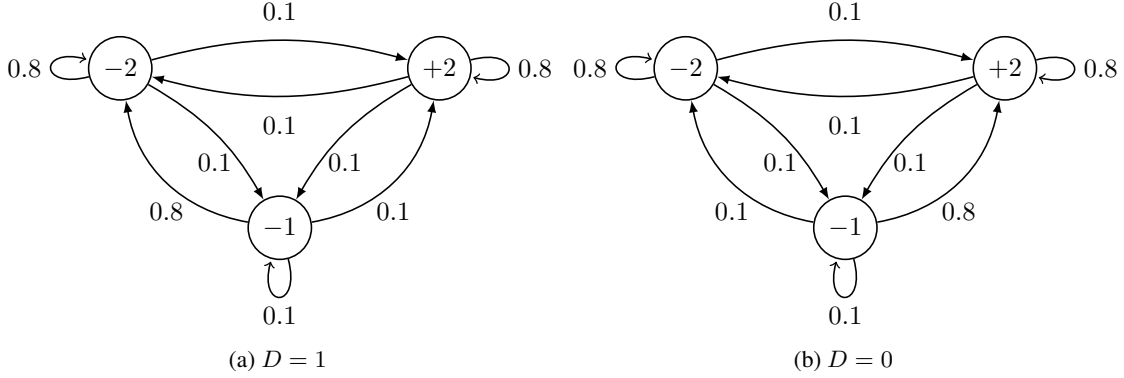

A detailed proof of Theorem \ref{th:um} is given in Appendix \ref{appndxproofth:util-max}. Below, we just provide a specific example that demonstrates the negative impact discussed in the statement of the theorem. This example is based on the "lack of motivation" concept identified by \cite{zhang_how_2020}, which refers to situations where institutional selection reduces the likelihood of qualification growth.
\begin{example}\label{example1}
We consider three possible qualifications: $-2$, $-1$ and $+2$, and their corresponding transition probabilities $q(y|y',d)$, are illustrated in Figure \ref{fig:tran}. Notably, the difference between the two decisions lies in the transition from state $-1$ to other states, where being selected increases the probability of experiencing a decline to qualification $-2$ and decreases the probability of growth.

We examine the groups $\A$ and $\B$ of equal size. Assuming a penalty function $g(x)=x$, we observe the impact of different levels of $\lambda$ on the disparity $\Delta$ over time, as illustrated in Figure \ref{fig:plot_discrimination}. The initial distribution of the two groups is as follows:
$p(Y_0=-2|\A)=0.3, p(Y_0=-1|\A)=0.1$, $p(Y_0=+2|\A)=0.6$ and
$p(Y_0=-2|\B)=0.5, p(Y_0=-1|\B)=0.1$, $p(Y_0=+2|\B)=0.4$
which we can show by the following vectors:
\begin{align*}
    p_{Y_0|\A} = [0.3, 0.1, 0.6],\\
    p_{Y_0|\B} = [0.5, 0.1, 0.4],
\end{align*}

When $\lambda$ falls within the range of $0$ to $0.5$ the penalty is not effective, as stated by Theorem \ref{th:effective}, and the disparity tends to decrease over time, approaching zero. On the other hand, when $\lambda$ exceeds $1$, the penalty is fully satisfactory, resulting in a constant disparity of zero. However, in the interval between 0.5 and 1, the penalty is effective but prevents equalization of qualifications and also increases the selection disparity. See Appendix \ref{appndxproofth:util-max} for further details.

\begin{figure}
    \centering
    \begin{tikzpicture}
    \begin{axis}[
      xlabel=Time,
      ylabel=$\Delta$,
      yticklabel style={
        /pgf/number format/fixed,
        /pgf/number format/precision=2
      }
    ]
    \addplot table [y=$00$, x=X]{data.dat};
    \addlegendentry{$\lambda = 0$}
    \addplot table [y=$07$, x=X]{data.dat};
    \addlegendentry{$\lambda = 0.7$}
    \addplot table [y=$15$, x=X]{data.dat};
    \addlegendentry{$\lambda = 1.5$}
    \end{axis}
    \end{tikzpicture}
    \caption{Evolution of disparity $\Delta$ under different levels of $\lambda$}
    \label{fig:plot_discrimination}
\end{figure}
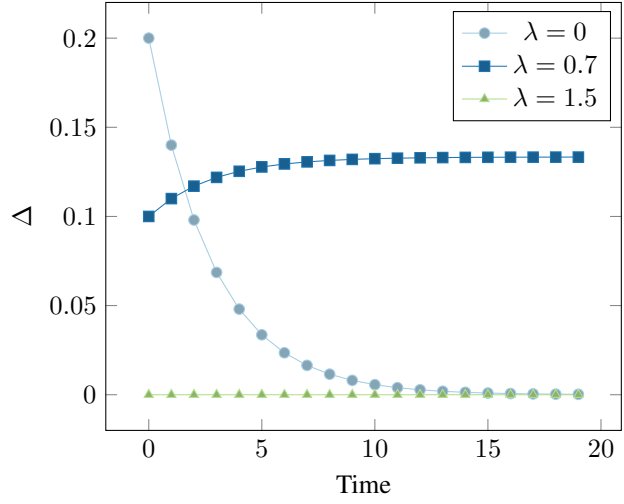

\end{example}

Observing the unintended consequence of the discrimination penalty makes it important to identify conditions that can prevent it. In the following theorem, we establish that if the transition probabilities $q(y|y',d)$ are sufficiently high, the disparity between groups will vanish regardless of the level of penalty imposed.

\begin{theorem}
    \label{th:penalty-conv}
    Define $\alpha = \min_{y,y',d} \{q(y|y',d)\}$,
    then for any $\lambda \geq 0$ and convex function $g(\cdot)$ we have
    \[
    \|p_{Y_{t+1}|\A}-p_{Y_{t+1}|\B}\|_1 \leq  2\left(1 - \alpha|\mathcal{Y}| \right)\|p_{Y_{t}|\A}-p_{Y_{t}|\B}\|_1.
    \]
    Specifically, $\frac{1}{2|\mathcal{Y}|}  < \alpha$ implies
    \[
\lim_{t\rightarrow\infty}\|p_{Y_{t+1}|\A}-p_{Y_{t+1}|\B}\|_1=0
    \]
    and
    \[
    \lim_{t\rightarrow\infty}\Delta_t=0.
    \]
\end{theorem}

Proof of Theorem \ref{th:penalty-conv} is given in Appendix \ref{appndxproofth:penalty-conv}.

\subsection{Dynamics of the institution's objective}

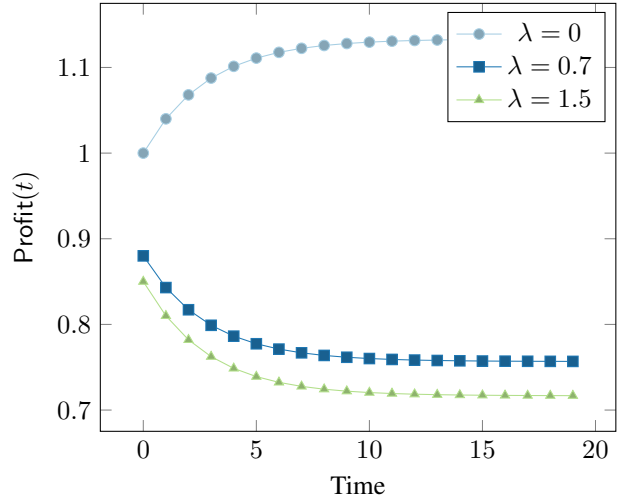
\begin{figure}
    \centering
        \begin{tikzpicture}
    \begin{axis}[
      xlabel=Time,
      ylabel=$\mathsf{Profit}(t)$,
      yticklabel style={
        /pgf/number format/fixed,
        /pgf/number format/precision=2
      }
    ]
    \addplot table [y=$00$, x=X]{data_util.dat};
    \addlegendentry{$\lambda = 0$}
    \addplot table [y=$07$, x=X]{data_util.dat};
    \addlegendentry{$\lambda = 0.7$}
    \addplot table [y=$15$, x=X]{data_util.dat};
    \addlegendentry{$\lambda = 1.5$}
    \end{axis}
    \end{tikzpicture}
    \caption{
    Evolution of the institution's objective $\mathsf{Profit}(t)=\E[D_tY_t]-\lambda\cdot g(\Delta_t)$ for different levels of discrimination penalty $\lambda$.}
    \label{fig:util-chart}
\end{figure}

In this section, we study the dynamics of the institution's objective $\mathsf{Profit}(t)=\E[D_tY_t]-\lambda\cdot g(\Delta_t)$ as time progresses. Figure \ref{fig:util-chart} considers the setting in Example \ref{example1} and plots the function $t\mapsto \mathsf{Profit}(t)$
for different values of $\lambda$. One can observe that $\mathsf{Profit}(t)$ might increase or decrease in time depending on the value of $\lambda$.
The following theorem studies the behavior of $\mathsf{Profit}(t)$ as a function of time for a different natural setting in which the selection of an individual has a positive average impact on one's qualification.

\begin{theorem}\label{thmnw5}
    Assume that selecting an individual would cause her average qualification to grow on average (for instance, due to the job experience), i.e.,
$$\mathbb{E}[Y_{t+1}|D_t=1, Y_t=y_t]\geq y_t, \qquad\forall y_t.$$
Equivalently,
$$\sum_{y_{t+1}}y_{t+1}q(y_{t+1}|D_t=1,y_t)\geq y_t, \qquad\forall y_t.$$
We do not make any assumption about the way qualifications of unselected individuals change over time. Then, for $\lambda\geq 0$, any penalty function $g(\cdot)$ and any initial distribution $p(c,y_0)$, the sequence of objective values at time $t$
\begin{gather}
\mathsf{Profit}(t)=\maximize_{p(d_t|c,y_t)}\; \left[\E[D_tY_t] - \lambda \cdot g(\F_t)\right]
\end{gather}
will be a non-decreasing function of $t$.
\end{theorem}

Proof of the above theorem can be found in Appendix \ref{appndxproofth:thmnw5}. Note that the objective function has two components $\E[D_tY_t]$ and $- \lambda \cdot g(\F_t)$. The fact that their sum will be a non-decreasing function of $t$ indicates that at least one of the components should not decrease in each time slot. However, this may not lead to social parity. Observe that if we have $\Pr(Y_0>0|C=\B)=0$ and $\lambda$ is not large enough for the penalty to become effective, the institution may never select from group $\B$ and one may never achieve social parity, even though the institution's total income (after paying the discrimination penalty) might still increase over time.

\subsection{Stationary states}

In this section, we are interested in stationary states (as defined in Definition \ref{def3}) in which groups have the same qualification distribution and, therefore, enjoy equal selection rates. Such stationary states always exist, as shown in the following lemma:

\begin{lemma} \label{lem:stationary}
    Consider a finite Markov chain with the transition matrix $T_{\mathcal{Y} \times \mathcal{Y}}$
\[
T(y'|y) =  \begin{cases}
        q(y'|y,1)& \text{if }y>0 \\
        q(y'|y,0)& \text{if }y<0
        \end{cases}, \qquad \forall y,y'\in\mathcal{Y}.
\]
   Let $r_{Y}$ be any stationary distribution of the above Markov chain. Then,
   $(p_{Y_t|\A}, p_{Y_t|\B})=(r_Y,r_Y)$ along with
  $$D= \begin{cases}
0,&Y<0\\
1,&Y>0
\end{cases}$$
   is a stationary state of our dynamic as defined in Definition \ref{def3}. This stationary distribution  satisfies
    $p_{Y_t|\A} = p_{Y_t|\B}$. Thus, for stationary states of this type, we have
     \[
    \Sct{C=\A} = \Sct{C=\B}.
    \]
\end{lemma}

Proof of Lemma \ref{lem:stationary} can be found in Appendix \ref{appndxproofth:band-matrix}. Note that, in general, the Markov chain $T$ does not necessarily have a unique stationary distribution. Therefore, our dynamic will not always have a unique stationary state with equal selection rates. However, finding cases where the stationary state is unique is desirable.

Next, we examine the case where qualification changes slowly over time, meaning that an individual can only move to the next higher or lower levels of qualification or stay put. In other words,
if $\mathcal{Y}=\{y^{(1)}, y^{(2)},\cdots, y^{(n)}\}$ such that
    \[
    y^{(1)} < y^{(2)} < \cdots < y^{(n)}
     \]
     we assume that
    \begin{align*}
        |i-j| \leq 1 \iff q(y^{(i)}|y^{(j)}, d) > 0.
    \end{align*}
One preliminary observation in this scenario is that stationary states of the type described in Lemma \ref{lem:stationary}
are unique in this case because the Markov chain in Lemma \ref{lem:stationary} will be irreducible and aperiodic. However, apart from this unique stationary state with equality across groups, other stationary states may exist that persist disparities.
But in the following theorem, we prove that if the selection increases the chance of qualification improvement and decreases the chance of degrading, then the stationary state from Lemma \ref{lem:stationary} is the unique stationary state of the system.

\begin{theorem} \label{th:band-matrix}
    Assume that
$\mathcal{Y}=\{y^{(1)}, y^{(2)},\cdots, y^{(n)}\}$ where $
    y^{(1)} < y^{(2)} < \cdots < y^{(n)}
     $.
     Moreover, assume that
    \begin{align}
        |i-j| \leq 1 \iff q(y^{(i)}|y^{(j)}, d) > 0\label{assumptionqt}
    \end{align}
and \begin{align}
    q(y^{(i+1)}|y^{(i)}, 1) \geq q(y^{(i+1)}|y^{(i)}, 0)\qquad
    \forall 1 \leq i < n
    \label{eqnaa71}\\
    q(y^{(i-1)}|y^{(i)}, 1) \leq q(y^{(i-1)}|y^{(i)}, 0),\qquad \forall 1  < i \leq n\label{eqnaa72}
\end{align}
Then, for any discrimination penalty $g(\cdot)$ and $\lambda$ the system has a unique stationary state. In this unique stationary distribution (which can be found using the procedure given in Lemma \ref{lem:stationary}), the groups have the same qualification distribution.
\end{theorem}
Proof of the above theorem can be found in Appendix \ref{appndxproofth:band-matrix}.

\section{Experiments} \label{sec:experiments}

We conducted experiments using real-world data from the Law School Admission Council (LSAC) dataset (\cite{Wightman1998LSACNL}). This dataset tracks the academic performance of 20,649 law students, encompassing sensitive attributes such as race. We consider admission to law school as a selection problem, with individual qualifications defined by their cumulative GPA. We subtracted 2.95 from the GPAs to obtain an individual's qualification, so a GPA below 3 will be negatively qualified and otherwise positive (GPAs are multiples of $0.1$ so subtracting $2.95$, instead of subtracting $3$, ensures that everyone gets a non-zero qualification). We categorized individuals into two distinct race groups consisting of 17921 white students and 3485 non-white students. The normalized histogram of the data in each group is depicted in Figure \ref{fig:lsac_init_dist}.

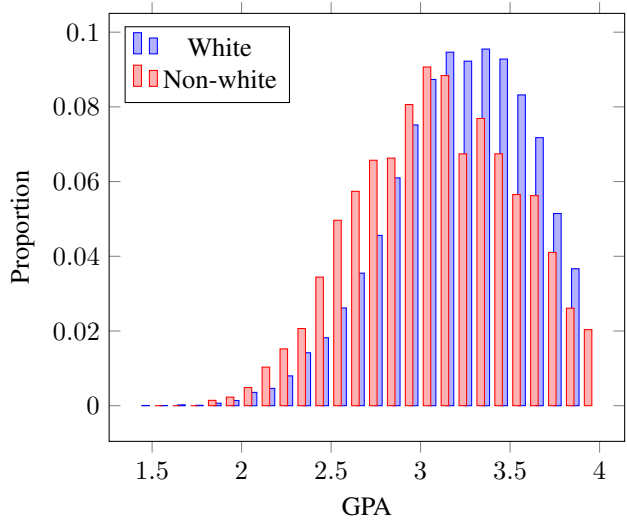
\begin{figure}
    \centering
\begin{tikzpicture}
  \begin{axis}[
      xlabel=GPA,
      ylabel=Proportion,
      ybar, 
      bar width=0.1cm,
      yticklabel style={
        /pgf/number format/fixed,
        /pgf/number format/precision=2
      },
      legend style={at={(0.03,0.97)},anchor=north west},
              ]
    \addplot+[fill] table[x=y,y=ap] {lsac_init.dat};
    \addlegendentry{White}
    \addplot+[fill] table[x=y,y=bp] {lsac_init.dat};
    \addlegendentry{Non-white}
     \end{axis}
\end{tikzpicture}
    \caption{Distribution of GPA across the two race groups in LSAC dataset}
    \label{fig:lsac_init_dist}
\end{figure}

Figure \ref{fig:lsac_effective} illustrates the impact of three different penalty functions $g(x)=x$, $g(x)=x^2$, and $g(x)=e^x$ on the selection disparity.
As predicted by Theorem \ref{th:effective} and confirmed by the experiments, we observe that for each penalty function, there is some threshold level of $\lambda$, below which, the penalty is ineffective in decreasing disparity in selection. Furthermore, since only $g(x)=x$ and $g(x)=e^x$ have non-zero derivatives at $x=0$, only these two functions can guarantee complete demographic parity ($\Delta=0$) for sufficiently large $\lambda$ in accordance with Theorem \ref{th:fulldp} and as depicted in the figure.

\begin{figure}
    \centering
\begin{tikzpicture}
\begin{axis}[
  xlabel=$\lambda$,
  ylabel=$\Delta$,
  yticklabel style={
    /pgf/number format/fixed,
    /pgf/number format/precision=2
  },
  xticklabel style={
    /pgf/number format/fixed,
    /pgf/number format/precision=2
  },
  legend style={nodes={scale=0.7, transform shape}},
  tick label style={font=\tiny},
  every axis plot/.append style={ultra thick}
]
\addplot+[no marks] table [y=x, x=lamb]{law_effective.dat};
\addlegendentry{$g(x)=x$}
\addplot+[no marks] table [y=x2, x=lamb]{law_effective.dat};
\addlegendentry{$g(x)=x^2$}

\addplot+[no marks] table [y=exp, x=lamb]{law_effective.dat};
\addlegendentry{$g(x)=e^x$}
\end{axis}
\end{tikzpicture}
\label{fig:chart}
    \caption{Impact of different penalty functions on disparity on LSAC dataset}
    \label{fig:lsac_effective}
\end{figure}
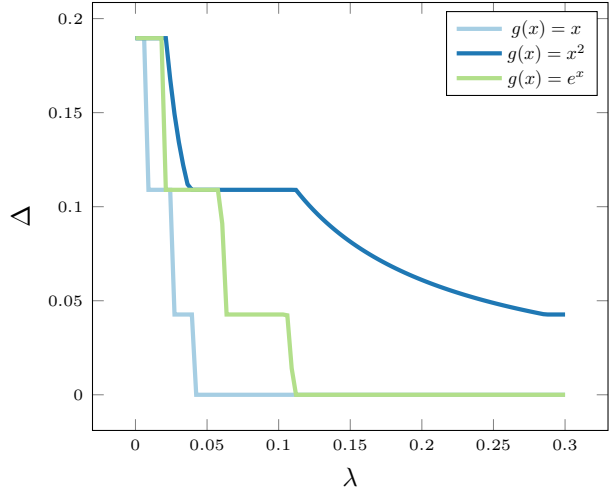

\subsection{Dynamics}

Since real-world datasets only capture a snapshot in time, we synthesize the dynamics for the experiments and build the initial state of the system based on the real-world data.
First, we consider dynamics satisfying the conditions of Theorem \ref{th:band-matrix}. Moreover, we assume that for selected individuals, the probabilities of growth and degradation are the same for all qualifications. In other words, for some $q^{1+}, q^{1-}, q^{0+}, q^{0-} \in [0,1]$ satisfying $q^{0+}+q^{0-}\leq 1$, $q^{1+}+q^{1-}\leq 1$ and $q^{0-}\geq q^{1-}$, we assume that for all $1 \leq i < n$ we have
\begin{align*}
q(y^{(i+1)}|y^{(i)},1) &= q^{1+},\\
     q(y^{(i-1)}|y^{(i)},1) &= q^{1-},
     \\
     q(y^{(i)}|y^{(i)},1) &= 1-q^{1+}-q^{1-},
\\
q(y^{(i+1)}|y^{(i)},0) &= q^{0+} ,\\
 q(y^{(i-1)}|y^{(i)},0)&= q^{0-},
 \\
 q(y^{(i)}|y^{(i)},0) &= 1-q^{0+}-q^{0-}.
\end{align*}

We generated 100 dynamics  by randomly choosing $q^{1+}, q^{1-}, q^{0+}, q^{0-} \in [0,1]$ as follows:
\begin{align*}
    q^{1+} &\sim \mathit{Uniform}(0, 1) \\
    q^{1-} &\sim \mathit{Uniform}(0, 1-q^{1+}) \\
    q^{0+} &\sim \mathit{Uniform}(0, q^{1+}) \\
    q^{0-} &\sim \mathit{Uniform}(q^{1-}, 1-q^{0+})
\end{align*}
Using a randomly chosen $\lambda \sim \mathit{Uniform}(0,1)$ for each dynamic, Figure \ref{fig:plot_homo_growth} depicts the quartiles of total variation between qualification distributions across the groups over time, showcasing the convergence $\lim_{t\rightarrow \infty} ||p_{Y_t|\A} - p_{Y_t|\B} ||_1= 0$ (which also implies $\lim_{t\rightarrow \infty} \Delta_t = 0$) across all samples. We conjecture that such a convergence holds in general for the dynamics considered in Theorem \ref{th:band-matrix}, and leave it for future work.

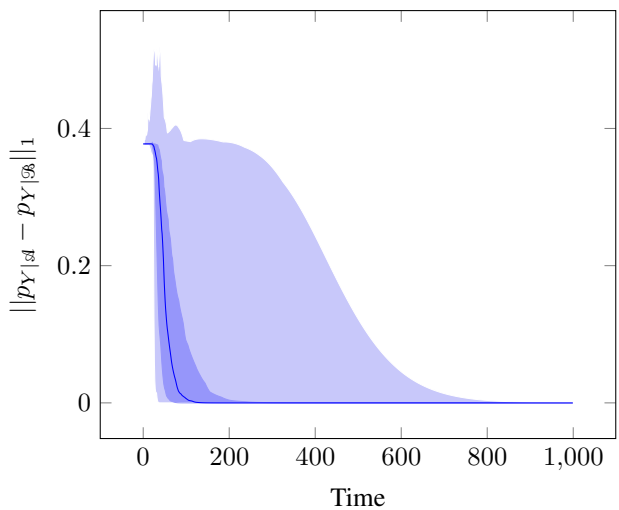
\begin{figure}
    \centering
        \begin{tikzpicture}
    \begin{axis}[
      xlabel=Time,
      ylabel=$||p_{Y|\A}-p_{Y|\B}||_1$,
      yticklabel style={
        /pgf/number format/fixed,
        /pgf/number format/precision=2
      }
    ]
    \addplot[color=white,name path=A] table [y=0.0, x=x]{homo_growth.dat};
    \addplot[color={rgb,255:red,150; green,150; blue,250},name path=B] table [y=0.25, x=x]{homo_growth.dat};
    \addplot[color={rgb,255:red,150; green,150; blue,250},name path=C] table [y=0.75, x=x]{homo_growth.dat};
    \addplot[color=white,name path=D] table [y=1.0, x=x]{homo_growth.dat};
    \addplot[color=blue,name path=M] table [y=0.5, x=x]{homo_growth.dat};
    \addplot+[color={rgb,255:red,200; green,200; blue,250}] fill between[of=A and D];
    \addplot+[color={rgb,255:red,150; green,150; blue,250}] fill between[of=B and C];
    \end{axis}
    \end{tikzpicture}
    \caption{Quartiles of $\Delta$ for 100 randomly generated dynamics satisfying conditions of Theorem} \ref{th:band-matrix}
    \label{fig:plot_homo_growth}
\end{figure}

Furthermore, in Theorem \ref{th:um} we constructed an example for $\lambda>0$ in which $\lim_{t\rightarrow\infty}\Delta_t>\Delta_0$. The question arises as to how common such examples are. To answer this question, we designed two simulations. We choose $g(x)=x$. To choose $\lambda$, we refer to Figure \ref{fig:lsac_effective}. We consider two choices of $\lambda=0.03$ and $\lambda=0.05$. Note that the penalty is fully satisfactory at $\lambda=0.05$, so in this case at time $t=0$, the penalty forces $\Delta_0=0$. However, for $\lambda=0.03$, we will have $\Delta_0>0$ at time $t=0$.

For $\lambda=0.03$, we first generated 20 random conditional distributions $q(y|y',d)$ and illustrated their dynamics
in Figure \ref{fig:random_band} (Figure \ref{fig:random_band_zoom} depicts a zoomed version of the same curves to show the initial dynamics). We observe that persistent disparity ($\lim_{t\rightarrow\infty}\Delta_t>\Delta_0$) occurs in 7 of the 20 samples. The same experiment is run for $\lambda=0.05$ in Figure \ref{fig:random_band_005} where
at time $t=0$, the penalty is fully satisfactory. Interestingly, while at time $t=0$, the penalty is fully satisfactory, as the underlying distribution of the qualifications of the two groups changes, the penalty may no longer be fully satisfactory and the disparity may emerge.

\begin{figure}
    \centering
    \begin{subfigure}{0.45\textwidth}
    \centering    \begin{tikzpicture}
    \begin{axis}[
      xlabel=Time,
      ylabel=$\Delta$,
      yticklabel style={
        /pgf/number format/fixed,
        /pgf/number format/precision=2
      },
      every axis plot/.append style={no marks, color=PuOr-F}
    ]
    \addplot table [y=1, x=x]{random_tran_band.dat};
    \addplot table [y=2, x=x]{random_tran_band.dat};
    \addplot table [y=3, x=x]{random_tran_band.dat};
    \addplot table [y=4, x=x]{random_tran_band.dat};
    \addplot table [y=5, x=x]{random_tran_band.dat};
    \addplot table [y=6, x=x]{random_tran_band.dat};
    \addplot table [y=7, x=x]{random_tran_band.dat};
    \addplot table [y=8, x=x]{random_tran_band.dat};
    \addplot table [y=9, x=x]{random_tran_band.dat};
    \addplot table [y=0, x=x]{random_tran_band.dat};
    \addplot table [y=11, x=x]{random_tran_band.dat};
    \addplot table [y=12, x=x]{random_tran_band.dat};
    \addplot table [y=13, x=x]{random_tran_band.dat};
    \addplot table [y=14, x=x]{random_tran_band.dat};
    \addplot table [y=15, x=x]{random_tran_band.dat};
    \addplot table [y=16, x=x]{random_tran_band.dat};
    \addplot table [y=17, x=x]{random_tran_band.dat};
    \addplot table [y=18, x=x]{random_tran_band.dat};
    \addplot table [y=19, x=x]{random_tran_band.dat};
    \addplot table [y=10, x=x]{random_tran_band.dat};
    \end{axis}
    \end{tikzpicture}
    \caption{Changes in $\Delta$ for $0\leq t\leq 1000$}
    \label{fig:random_band}
    \end{subfigure}
    \hfill
    \begin{subfigure}{0.45\textwidth}
    \centering    \begin{tikzpicture}
    \begin{axis}[
      xlabel=Time,
      ylabel=$\Delta$,
      yticklabel style={
        /pgf/number format/fixed,
        /pgf/number format/precision=2
      },
      every axis plot/.append style={no marks, color=PuOr-F}
    ]
    \addplot table [y=1, x=x, restrict x to domain=0:50]{random_tran_band.dat};
    \addplot table [y=2, x=x, restrict x to domain=0:50]{random_tran_band.dat};
    \addplot table [y=3, x=x, restrict x to domain=0:50]{random_tran_band.dat};
    \addplot table [y=4, x=x, restrict x to domain=0:50]{random_tran_band.dat};
    \addplot table [y=5, x=x, restrict x to domain=0:50]{random_tran_band.dat};
    \addplot table [y=6, x=x, restrict x to domain=0:50]{random_tran_band.dat};
    \addplot table [y=7, x=x, restrict x to domain=0:50]{random_tran_band.dat};
    \addplot table [y=8, x=x, restrict x to domain=0:50]{random_tran_band.dat};
    \addplot table [y=9, x=x, restrict x to domain=0:50]{random_tran_band.dat};
    \addplot table [y=0, x=x, restrict x to domain=0:50]{random_tran_band.dat};
    \addplot table [y=11, x=x, restrict x to domain=0:50]{random_tran_band.dat};
    \addplot table [y=12, x=x, restrict x to domain=0:50]{random_tran_band.dat};
    \addplot table [y=13, x=x, restrict x to domain=0:50]{random_tran_band.dat};
    \addplot table [y=14, x=x, restrict x to domain=0:50]{random_tran_band.dat};
    \addplot table [y=15, x=x, restrict x to domain=0:50]{random_tran_band.dat};
    \addplot table [y=16, x=x, restrict x to domain=0:50]{random_tran_band.dat};
    \addplot table [y=17, x=x, restrict x to domain=0:50]{random_tran_band.dat};
    \addplot table [y=18, x=x, restrict x to domain=0:50]{random_tran_band.dat};
    \addplot table [y=19, x=x, restrict x to domain=0:50]{random_tran_band.dat};
    \addplot table [y=10, x=x, restrict x to domain=0:50]{random_tran_band.dat};
    \end{axis}
    \end{tikzpicture}
    \caption{Changes in $\Delta$ for $0\leq t\leq 50$)}
    \label{fig:random_band_zoom}
        \end{subfigure}
\caption{Evolution of $\Delta$ for random transitions with $\lambda=0.03$}

\end{figure}
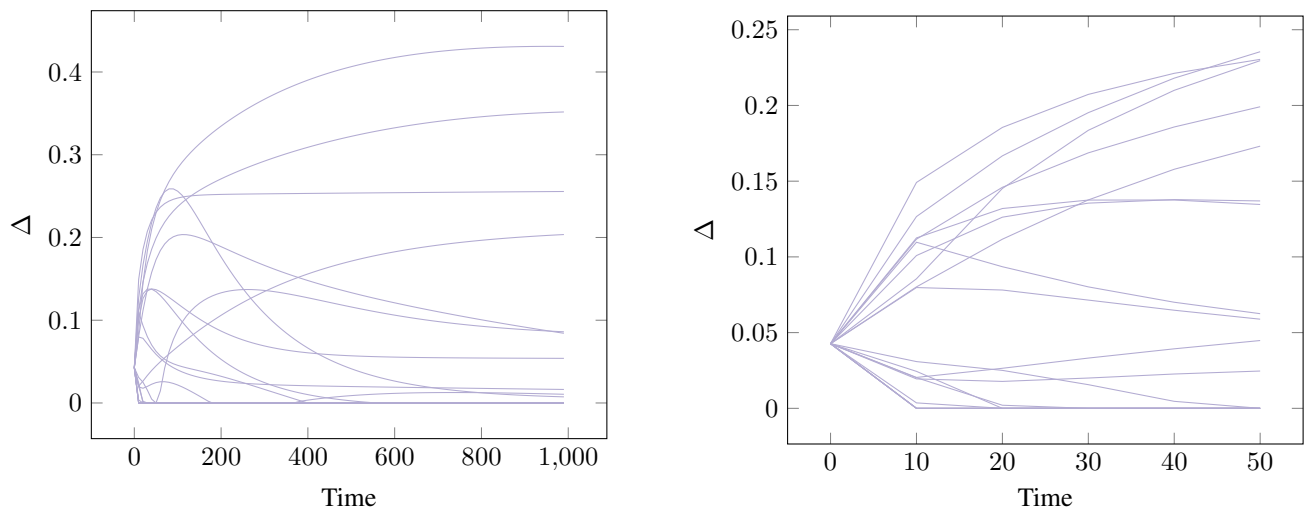

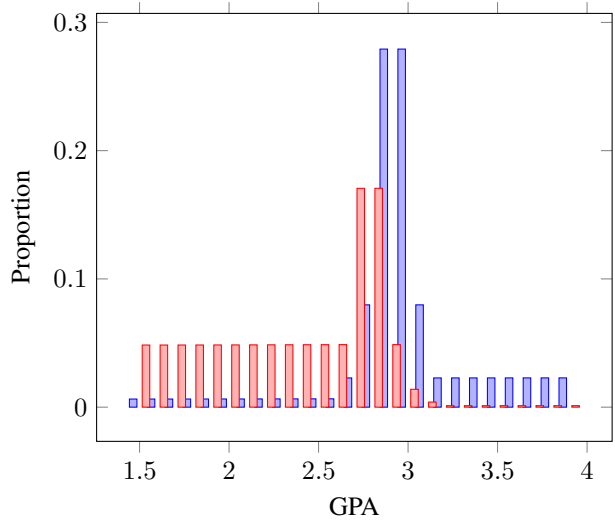
\begin{figure}
    \centering
\begin{tikzpicture}
  \begin{axis}[
      xlabel=GPA,
      ylabel=Proportion,
      ybar, 
      bar width=0.1cm,
      yticklabel style={
        /pgf/number format/fixed,
        /pgf/number format/precision=2
      },
      legend style={at={(0.03,0.97)},anchor=north west},
              ]
    \addplot+[fill] table[x=y,y=a] {disparity_final_dist.dat};
    \addplot+[fill] table[x=y,y=b] {disparity_final_dist.dat};
     \end{axis}
\end{tikzpicture}
    \caption{Final distribution of GPA across the groups in a dynamic leading to persisting disparity}
    \label{fig:disparity_final_dist}
\end{figure}
\begin{figure}
    \centering
    \begin{tikzpicture}
    \begin{axis}[
      xlabel=Time,
      ylabel=$\Delta$,
      yticklabel style={
        /pgf/number format/fixed,
        /pgf/number format/precision=2
      },
      every axis plot/.append style={no marks, color=PuOr-F}
    ]
    \addplot table [y=1, x=x]{random_tran_band_005.dat};
    \addplot table [y=2, x=x]{random_tran_band_005.dat};
    \addplot table [y=3, x=x]{random_tran_band_005.dat};
    \addplot table [y=4, x=x]{random_tran_band_005.dat};
    \addplot table [y=5, x=x]{random_tran_band_005.dat};
    \addplot table [y=6, x=x]{random_tran_band_005.dat};
    \addplot table [y=7, x=x]{random_tran_band_005.dat};
    \addplot table [y=8, x=x]{random_tran_band_005.dat};
    \addplot table [y=9, x=x]{random_tran_band_005.dat};
    \addplot table [y=0, x=x]{random_tran_band_005.dat};
    \addplot table [y=11, x=x]{random_tran_band_005.dat};
    \addplot table [y=12, x=x]{random_tran_band_005.dat};
    \addplot table [y=13, x=x]{random_tran_band_005.dat};
    \addplot table [y=14, x=x]{random_tran_band_005.dat};
    \addplot table [y=15, x=x]{random_tran_band_005.dat};
    \addplot table [y=16, x=x]{random_tran_band_005.dat};
    \addplot table [y=17, x=x]{random_tran_band_005.dat};
    \addplot table [y=18, x=x]{random_tran_band_005.dat};
    \addplot table [y=19, x=x]{random_tran_band_005.dat};
    \addplot table [y=10, x=x]{random_tran_band_005.dat};
    \end{axis}
    \end{tikzpicture}
    \caption{Evolution of $\Delta$ for random transitions with $\lambda=0.05$}
    \label{fig:random_band_005}
\end{figure}
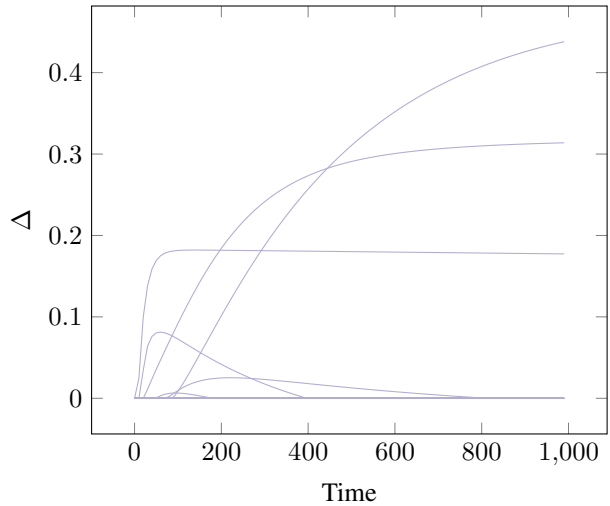

Next, as another experiment, considering the initial configuration derived from the dataset, we first manually identified a dynamic $q(y|y',d)$ leading to persistent disparities ($\lim_{t\rightarrow\infty}\Delta_t>\Delta_0$) when $g(x)=x$ and $\lambda=0.03$. The final distribution of qualifications across groups in this dynamic is depicted in Figure \ref{fig:disparity_final_dist}. Subsequently, we generated 10 random perturbations of $q(y|y',d)$ by adding Gaussian noise of $\mathcal{N}(0, \sigma^2=0.1)$ to each transition probability, then taking the absolute value and normalizing to obtain a valid $\tilde q(y|y',d)$ again. Note that the noise standard deviation $\sigma\approx 0.31$, so the perturbations are not small.
Figure \ref{fig:samles_persisting_disparity} reveals that, aside from one instance, all systems tend towards persistent disparities over time, i.e., $\lim_{t\rightarrow\infty}\Delta_t>\Delta_0$.  We repeated the experiment of Figure \ref{fig:lsac_persisting_disparity_005} for $\lambda=0.05$.

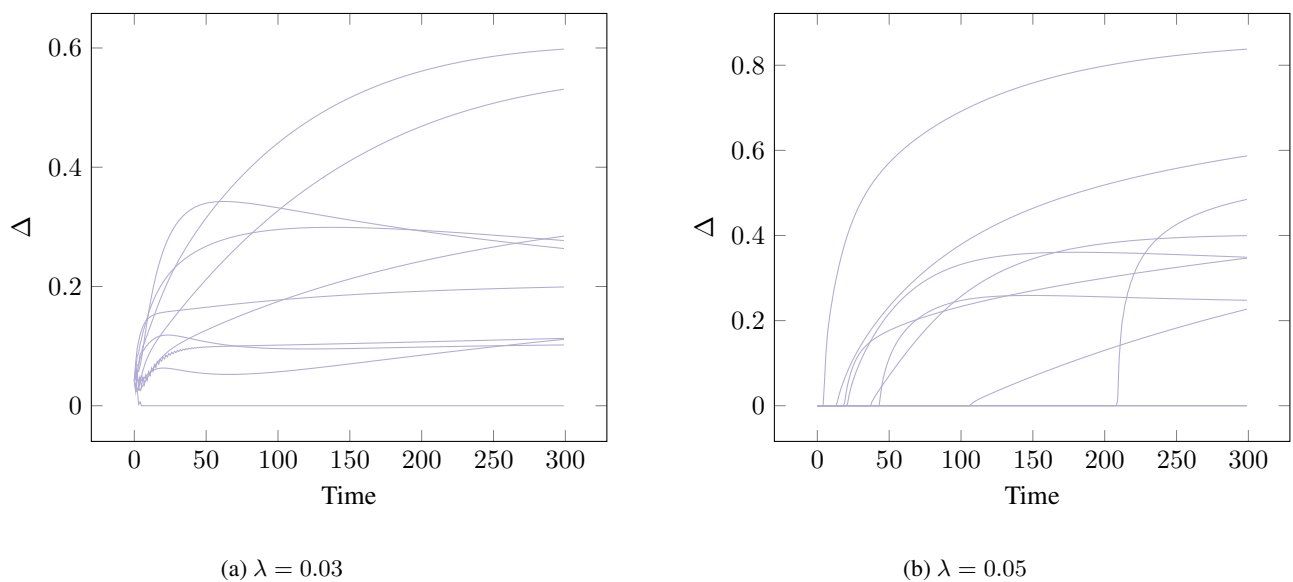
\begin{figure}
    \centering
\begin{subfigure}{.45\textwidth}
    \centering
        \begin{tikzpicture}
    \begin{axis}[
      xlabel=Time,
      ylabel=$\Delta$,
      yticklabel style={
        /pgf/number format/fixed,
        /pgf/number format/precision=2
      },
      every axis plot/.append style={no marks, color=PuOr-F}
    ]
    \addplot[name path=A] table [y=1, x=x]{disparity_perturb.dat};
    \addplot[name path=A] table [y=2, x=x]{disparity_perturb.dat};
    \addplot[name path=A] table [y=3, x=x]{disparity_perturb.dat};
    \addplot[name path=A] table [y=4, x=x]{disparity_perturb.dat};
    \addplot[name path=A] table [y=5, x=x]{disparity_perturb.dat};
    \addplot[name path=A] table [y=6, x=x]{disparity_perturb.dat};
    \addplot[name path=A] table [y=7, x=x]{disparity_perturb.dat};
    \addplot[name path=A] table [y=8, x=x]{disparity_perturb.dat};
    \addplot[name path=A] table [y=9, x=x]{disparity_perturb.dat};
    \addplot[name path=A] table [y=0, x=x]{disparity_perturb.dat};
    \end{axis}
    \end{tikzpicture}
    \caption{$\lambda=0.03$}\label{fig:samles_persisting_disparity}
\end{subfigure}
\hfill
\begin{subfigure}{.45\textwidth}
    \centering
        \begin{tikzpicture}
    \begin{axis}[
      xlabel=Time,
      ylabel=$\Delta$,
      yticklabel style={
        /pgf/number format/fixed,
        /pgf/number format/precision=2
      },
      every axis plot/.append style={no marks, color=PuOr-F}
    ]
    \addplot[name path=A] table [y=1, x=x]{disparity_perturb_005.dat};
    \addplot[name path=A] table [y=2, x=x]{disparity_perturb_005.dat};
    \addplot[name path=A] table [y=3, x=x]{disparity_perturb_005.dat};
    \addplot[name path=A] table [y=4, x=x]{disparity_perturb_005.dat};
    \addplot[name path=A] table [y=5, x=x]{disparity_perturb_005.dat};
    \addplot[name path=A] table [y=6, x=x]{disparity_perturb_005.dat};
    \addplot[name path=A] table [y=7, x=x]{disparity_perturb_005.dat};
    \addplot[name path=A] table [y=8, x=x]{disparity_perturb_005.dat};
    \addplot[name path=A] table [y=9, x=x]{disparity_perturb_005.dat};
    \addplot[name path=A] table [y=0, x=x]{disparity_perturb_005.dat};
    \end{axis}
    \end{tikzpicture}
    \caption{$\lambda=0.05$}\label{fig:lsac_persisting_disparity_005}
\end{subfigure}
\caption{Evolution of $\Delta$ for 10 random perturbations of a dynamic resulting in persistent disparity}
\end{figure}

\section{Conclusion}

In addressing the impact of regulating AI discrimination within society, it is essential to undertake a thorough assessment of the efficiency of potential interventions. Approaches driven by intuitive motives may prove ineffective or, worse, lead to unintended consequences.

In this paper, we introduced a simple model that encapsulates the strategy of penalizing discrimination among various groups in the selection process of utility-maximizing institutions. Initially, we established conditions necessary for the effectiveness of such penalties, (i.e., compel the institution to reduce the disparity in selection from their initial state), and conditions for them to be fully satisfactory (i.e., reducing the disparity to zero). These conditions illustrate potential scenarios where penalties may prove ineffective and where different penalty formulas may yield disparate outcomes.
Subsequently, we examined the evolution of individuals' qualifications with respect to institutions' selection through a stochastic process model. We demonstrated that equal opportunities for change across groups, naturally lead to parity in qualifications and selection over time. However, we highlighted counterintuitive cases where penalizing discrimination may prevent this convergence. Next, we identified conditions that ensure the prevention of such scenarios, including sufficiently high transition probabilities and a positive impact of selection in cases of gradual change.

\bibliographystyle{unsrtnat}
\bibliography{references}

\appendix

\section{\label{app:proof}}
\subsection{A useful lemma}
\begin{lemma}
\label{lem1}
    If $\lambda>0$, the optimal policy $p(d|c,y)$ maximizing \eqref{eq:prob} is not necessarily unique. Moreover, each optimal policy $p(d|c,y)$ satisfies the following properties:
    \begin{enumerate}
        \item Selection probability within a group is a non-decreasing function of qualification ($y$). More precisely,   if 
    $\Pr(Y=y,C=c)>0$ and $\Pr(Y=y',C=c)>0$ for some $y>y'$ and $c \in \G$, then 
            \[\Sc{Y=y, C=c} \geq\Sc{Y=y', C=c}
            \]
        \item Within each group at most one qualification subgroup has fractional selection selection probability, i.e.,
            \[
            \forall c \,\exists y' \in \mathcal{Y}: \;  \; \Sc{Y=y, C=c} \in \{0,1\}\qquad\forall y \neq y': \Pr(Y=y,C=c)>0.
            \]
        \item Assume $\Pr(Y>0|C=\A) \geq \Pr(Y>0|C=\B)$.\footnote{This assumption can be made without loss of generality in the static case. See \eqref{assumptione}.} The advantaged group, $\A$, never has a smaller selection rate:
        \[\Sc{C=\A} \geq \Sc{C=\B} \]
        \item Positively qualified individuals from group $\B$ are always selected and negatively qualified individuals from group $\A$ never get selected, i.e., 
            \begin{gather*}
            \Pr(Y=y, C=\A)>0,~~~ y < 0 \implies \Sc{Y=y, C=\A} = 0 \\
            \Pr(Y=y, C=\B)>0,~~~ y > 0 \implies \Sc{Y=y, C=\B} = 1
            \end{gather*}
    \end{enumerate}
\end{lemma}

\begin{proof}$ $ \par

\begin{enumerate}
    \item 
   If for group $c$ and qualifications $y$ and $y'$
\[
y > y' \implies \Sc{y, C=c} \geq\Sc{y', C=c}
\]
does not hold, we can increase $\Sc{y, C=c}$ and decrease $\Sc{y', C=c}$ while fixing $\Sc{C=c}$. This way $\F$ will not change, but $\U=\E[DY]$ will increase.

    \item Assume that 
\[
0<\Sc{y, C=c}<1 ,\quad  0<\Sc{y', C=c} <1
\]
for some $y > y'$. Increase $\Sc{y, C=c}$ and decrease $\Sc{y', C=c}$ while fixing $\Sc{C=c}$. This way $\F$ will not change, but $\U=\E[DY]$ will increase.

    \item Assume
\[
\Sc{C=\A} < \Sc{C=\B}.
\]
Since for $c \in \G$ we have
\[
\Sc{y, C=c} = \Pr(Y>0|C=c) \cdot \Sc{Y > 0, C=c} + \Pr(Y\leq 0|C=c) \cdot \Sc{Y \leq 0, C=c}
\]
then $\Sc{C=\A} < \Sc{C=\B}$ implies either
\begin{align}
\Pr(Y>0|C=\A) \cdot \Sc{Y > 0, C=\A} < \Pr(Y>0|C=\B) \cdot \Sc{Y > 0, C=\B}\label{eqnT1}
\end{align}
or
\begin{align}
\Pr(Y\leq 0|C=\A) \cdot \Sc{Y \leq 0, C=\A} < \Pr(Y \leq 0|C=\B) \cdot \Sc{Y \leq 0, C=\B}.
\label{eqnT2}
\end{align}
Consider the first case. The assumption $\Pr(Y>0|C=\A) \geq \Pr(Y>0|C=\B)$ in conjunction with \eqref{eqnT1} implies $\Sc{Y>0, C=\A} < \Sc{Y>0, C=\B}$. Thus, $\Sc{Y>0, C=\A}<1$. So, at least for one $y 
> 0$ we have $\Sc{y, C=\A} < 1$ and $\Pr(Y=y|C=\A)>0$.
Select $\epsilon \in \R$ such that 
\[
0 < \epsilon < \min\left\{\frac{\Sc{C=\B} - \Sc{C=\A}}{\Pr(Y=y|C=\A)}, \, 1 - \Sc{y, C=\A}\right\}
\]
If the institution increases $\Sc{y, C=\A}$ by $\epsilon$, the inequality $\Sc{C=\A} < \Sc{C=\B}$ will stay satisfied, and the discrimination 
$\Sc{C=\B}-\Sc{C=\A}$ will decrease.
Moreover, $\U=\E[DY]$ increases. This is at odds with the optimality of the policy.

The second case is similar. Assume that \eqref{eqnT2} holds. 
Then, we deduce $0 < \Sc{Y \leq 0, C=\B}.$ So, there exists $y < 0$ such that $\Sc{y, C=\B} > 0$ and $\Pr(Y=y|C=\B)>0$. 
Select $\epsilon \in \R$ such that 
\[
0 < \epsilon < \min\left\{\frac{\Sc{C=\B} - \Sc{C=\A}}{\Pr(Y=y|C=\B)}, \, \Sc{y, C=\B}\right\}
\]
Now, the institution can increase its utility and decrease the penalty by decreasing $\Sc{y, C=\B}$ by $\epsilon$.

\item Assume $\Sc{Y=y, C=\A} > 0$ for some $y < 0$. From the previous part, we know that $\Sc{C=\A} \geq \Sc{C=\B}$. If $\Sc{C=\A} > \Sc{C=\B}$ we can decrease $\Sc{Y=y, C=\A}$ slightly, so that $u=\E[DY]$ increases and $\Delta$ decreases. So the assumed policy is not optimal. Next, consider the case of $\Sc{C=\A} = \Sc{C=\B} $. The assumption $\Sc{Y=y, C=\A} > 0$ and parts 1 and 2 of the lemma imply that for any $y'>y$ such that $\Pr(Y=y'|C=\A)>0$ we have
\[
\Sc{y', C=\A} = 1
\]
So,
\[
\Sc{C=\B} = \Sc{C=\A} \geq \Pr(Y>0|C=\A) > \Pr(Y>0|C=\B).
\]
This, in turn, yields that there is some $y' < 0$ such that $\Sc{Y=y', C=\B} > 0$. By decreasing $\Sc{Y=y, C=\A}$ and $\Sc{Y=y', C=\B}$ appropriately, we can increase $u=\E[DY]$ while keeping $\Delta$ unchanged. 

Next assume $\Sc{Y=y, C=\B} < 1$ for some $y > 0$. This time, if $\Sc{C=\A} > \Sc{C=\B}$ we can increase $\Sc{Y=y, C=\B}$ slightly to improve the institution's objective. If $\Sc{C=\A} = \Sc{C=\B}$, parts 1 and 2 of Lemma \ref{lem1} imply that for all $y' < y$ we have 
\[
\Sc{Y=y', C=\B}= 0.
\]
In other words, 
$\Sc{Y=y', C=\B}> 0
$ implies that $y'>0$. 
This implies that $\Sc{C=\B} \leq \Pr(Y>0|C=\B)$. Then, using the  assumption $\Pr(Y>0|C=\A) >\Pr(Y>0|C=\B)$ in \eqref{assumptione} we have
\[
\Sc{C=\A} = \Sc{C=\B} \leq \Pr(Y>0|C=\B) < \Pr(Y>0|C=\A).
\]
From $\Sc{C=\A}< \Pr(Y>0|C=\A)$ we conclude that there is some $y'' > 0$ such that $\Sc{Y=y'', C=\A} < 1$. 
We can increase $\Sc{Y=y'', C=\A}$ and $\Sc{Y=y, C=\B}$ simultaneously to improve the objective while keeping $\Delta$ unchanged.

\end{enumerate}
\end{proof}

Next, we introduce an equivalent formulation for the main optimization problems in \eqref{eq:prob} and \eqref{eq:demo-parity}. This formulation simplifies the proof for Theorem \ref{th:effective} and Theorem \ref{th:fulldp}. It relates the optimization problem for any $\lambda>0$ to $\Uum=\E[Y_+]$, which is the value of the optimization problem for $\lambda=0$ (see \eqref{uumdef}).

\begin{lemma}\label{lemma-represent}
The optimization problem in \eqref{eq:prob} has the following equivalent form:
\begin{align}
\max_{p(d|c,y)}\; \E[DY] - \lambda \cdot g(\F)
&=
    \label{eq:prob-simp}
    \max_{z_{yc}}&\; \E[Y_+] - \sum_{(y,c) \in \mathcal{T}_e} z_{yc}\cdot p(c)\cdot |y| - \lambda \cdot g\left(\Delnp - \sum_{(y,c) \in \mathcal{T}_e} z_{yc}\right) 
\end{align}

where $\Delnp$ is defined in \eqref{defDNp}, the set $\mathcal{T}_e$ is defined as follows 
\begin{align}\mathcal{T}_e=\{(y,\A): y>0, p(y|C=\A)>0\}\cup \{(y,\B): y<0, p(y|C=\B)>0\},\label{defTe}\end{align}
and the maximum is over $z_{yc}$ (for $c\in \G$) satisfying 
\begin{align}\sum_{(y,c) \in \mathcal{T}_e} z_{yc} &\leq \Delnp, \label{constd1}\\
    0 \leq z_{yc} &\leq p(y|c) \qquad \forall y,c.\label{constd2}\end{align}
Similarly, the optimization problem in \eqref{eq:demo-parity} has the following equivalent form:
\begin{align} 
&\max_{
\substack{p(d|c,y):  \\ p(D=1|C=\A) = p(D=1|C=\B)}
}\; \E[DY] 
= \label{eq:demo-parity-simp}
    \max_{z_{yc}}\; \E[Y_+] - \sum_{(y,c) \in \mathcal{T}_e} z_{yc}\cdot p(c)\cdot |y| 
\end{align}
where the maximum is over $z_{yc}$ satisfying 
\begin{align}
   \sum_{(y,c) \in \mathcal{T}_e} z_{yc} &= \Delnp, \label{const1p2}\\
    0 \leq z_{yc} &\leq p(y|c), \qquad \forall y,c. \label{const2p2}
\end{align}
\end{lemma}
\begin{proof}

Let $p^*(d|y,c)$ be any optimal solution of the optimization problem in \eqref{eq:prob}. Define
\begin{align}
    \label{eq:substitute}
    &z_{y\A} = p(y|\A) \left(1 - p^*(D=1|Y=y, C=\A)\right) \\
    &z_{y\B} = p(y|\B) p^*(D=1|Y=y, C=\B)
    .\nonumber
\end{align}
Applying the fourth part of Lemma \ref{lem1}, then we can state the institution's utility under this solution as 
\begin{align*}
\E_{p^*}[YD] &= \sum_{y>0} y\cdot p(\A) \cdot p(y|\A) - \sum_{y>0} |y| \cdot p(\A)\cdot z_{y\A} + \sum_{y>0} y\cdot p(\B) \cdot p(y|\B) - \sum_{y<0} |y|\cdot  p(\B)\cdot z_{y\B}  \\
&= \E[Y_+] - \sum_{y<0} |y|\cdot  p(\B)\cdot z_{y\B} - \sum_{y>0} |y| \cdot p(\A)\cdot z_{y\A}
\end{align*}
Using \eqref{defDelta2} and the fourth part of Lemma \ref{lem1}, the disparity can be written as
\begin{align*}
    \Delta_{p^*} = \sum_{y>0}  p(y|\A)- \sum_{y>0} z_{y\A} - \sum_{y>0} p(y|\B) - \sum_{y<0} z_{y\B}  = \Delnp  - \sum_{y>0} z_{y\A}- \sum_{y<0} z_{y\B}.
\end{align*}
Note that the above inequality implies that
\begin{align}
    \sum_{y>0} z_{y\A}+\sum_{y<0} z_{y\B}\leq \Delnp.\label{eqnNNT}
\end{align}
We would like to apply a change of variables and write the optimization problem in \eqref{eq:prob} in terms of $z_{y\A}$ for $y>0$, and $z_{y\B}$ for $y<0$. More specifically, we use variables $z_{yc}$ where $(y,c)$ is in $\mathcal{T}_e$:
$$\mathcal{T}_e=\{(y,\A): y>0, p(y|C=\A)>0\}\cup \{(y,\B): y<0, p(y|C=\B)>0\}.$$
We claim that any $z_{yc}$ (for $(y,c) \in \mathcal{T}_e$) satisfying \eqref{constd1} and \eqref{constd2}
corresponds to some decision rule $p(d|y,c)$. Conversely, any \emph{optimal} decision rule $p^*(d|y,c)$ corresponds to some $z_{yc}$ satisfying \eqref{constd1} and \eqref{constd2}. The latter can be verified from 
\eqref{eqnNNT} and the definition given in \eqref{eq:substitute}.

Take some $z_{yc}$ satisfying \eqref{constd1} and \eqref{constd2}, and define
\begin{align}
    \label{eq:substitute2}
    p(D=1|Y=y, C=\A)=&1-\frac{z_{y\A}}{p(y|\A)} ,  &\forall y>0: p(y|C=\A)>0, \\
    p(D=1|Y=y, C=\B)=&\frac{z_{y\B} }{p(y|\B)}, &\forall y<0: p(y|C=\B)>0.\nonumber
\end{align}
Moreover, set $p(D=1|Y=y, C=\A)=0$ for $y<0$ and $p(D=1|Y=y, C=\B)=1$ for $y>0$. Since for all $d,y$ and $c$ we have $0 \leq p(D=d|Y=y,C=c) \leq 1$, this constitutes a valid decision rule. 

Thus, the optimization problem in \eqref{eq:prob} is equivalent to
\begin{align}
\maximize&\; \Uum - \sum_{(y,c) \in \mathcal{T}_e} z_{yc}\cdot p(c)\cdot |y| - \lambda \cdot g\left(\Delnp - \sum_{(y,c) \in \mathcal{T}_e} z_{yc}\right) 
\end{align}
where the maximum is over $z_{yc}$ satisfying \eqref{constd1} and \eqref{constd2}.

Next, consider the second part of the theorem. We have
\begin{align} 
&\max_{
\substack{p(d|c,y):  \\ p(D=1|C=\A) = p(D=1|C=\B)}
}\; \E[DY] = \max_{p(d|c,y)}\inf_{\lambda\geq 0}\; \E[DY] - \lambda \cdot g(\F)
\end{align}
Note that $\E[DY] - \lambda \cdot g(\F)$ is concave in $p(d|c,y)$ for any fixed $\lambda$. This follows from the fact that  $x\mapsto |x|$ is a convex function and $g(\cdot)$ is a non-decreasing convex function. Next, 
$\E[DY] - \lambda \cdot g(\F)$ is linear in $\lambda$ for any fixed $p(d|c,y)$. Moreover, the domain of  $p(d|c,y)$ is a compact set. Therefore, Sion's minimax theorem (\cite{sion_minimax}) implies that  
\begin{align*}
\max_{p(d|c,y)}\inf_{\lambda\geq 0}\; \E[DY] - \lambda \cdot g(\F)
&=
\inf_{\lambda\geq 0}\max_{p(d|c,y)}\; \E[DY] - \lambda \cdot g(\F).
\end{align*}
We can now utilize the first part of the theorem to deduce that
\begin{align*}
\inf_{\lambda\geq 0}\max_{p(d|c,y)}\; \E[DY] - \lambda \cdot g(\F)=
\inf_{\lambda\geq 0}\max_{z_{yc}}\; \E[Y_+] - \sum_{(y,c) \in \mathcal{T}_e} z_{yc}\cdot p(c)\cdot |y| - \lambda \cdot g\left(\Delnp - \sum_{(y,c) \in \mathcal{T}_e} z_{yc}\right) 
\end{align*}
where the maximum is over $z_{yc}$ satisfying \eqref{constd1} and \eqref{constd2}. 
Using a similar minimax exchange, we have
\begin{align*}
&\inf_{\lambda\geq 0}
    \max_{z_{yc}}\; \E[Y_+] - \sum_{(y,c) \in \mathcal{T}_e} z_{yc}\cdot p(c)\cdot |y| - \lambda \cdot g\left(\Delnp - \sum_{(y,c) \in \mathcal{T}_e} z_{yc}\right) 
\\&=
\max_{z_{yc}}\inf_{\lambda\geq 0}\; \E[Y_+] - \sum_{(y,c) \in \mathcal{T}_e} z_{yc}\cdot p(c)\cdot |y| - \lambda \cdot g\left(\Delnp - \sum_{(y,c) \in \mathcal{T}_e} z_{yc}\right) 
\\&=
\max_{z_{yc}: \Delnp =\sum_{(y,c) \in \mathcal{T}_e} z_{yc}} \E[Y_+] - \sum_{(y,c) \in \mathcal{T}_e} z_{yc}\cdot p(c)\cdot |y| 
\end{align*}
This completes the proof.
\end{proof}

\begin{lemma}\label{lemmanew4}
    Consider the equivalent form in \eqref{eq:demo-parity-simp} for the problem of finding the best policy that satisfies demographic parity:
\begin{align} 
&\max_{z_{yc}}\; \E[Y_+] - \sum_{(y,c) \in \mathcal{T}_e} z_{yc}\cdot p(c)\cdot |y| 
\end{align}
where the maximum is over $z_{yc}$ satisfying 
\begin{align}
   \sum_{(y,c) \in \mathcal{T}_e} z_{yc} &= \Delnp,\\
    0 \leq z_{yc} &\leq p(y|c), \qquad \forall y,c. 
\end{align}
where $\mathcal{T}_e$ is defined in \eqref{defTe}. 
Then, there is some constant $\beta_s$ such that any maximizer satisfies
$$\max_{(c,y) \in \mathcal{T}_e: z_{yc} > 0} p(c)\cdot|y|=\beta_s.
$$
Moreover, any maximizer $z_{yc}$ satisfies the following property:
\begin{align*}
    & p(c).|y| > \beta_s \implies z_{yc}=0,\\
& p(c).|y| < \beta_s \implies z_{yc} = p(y|c),
\end{align*}
implying that two maximizers can differ only on pairs $(c,y)$ where $p(c).|y|=\beta_s$.

\end{lemma}
\begin{proof}
We begin by showing that for any arbitrary $(y,c)$ and $(y',c')$ in $\mathcal{T}_e$ such that $z_{yc}>0$ and $p(c')\cdot |y'| < p(c)\cdot |y|$ we have 
$z_{y'c'} = p(y'|c')$.
This property must hold, since otherwise we can choose some $0 < \epsilon < z_{yc}$ and decrease $z_{yc}$ by $\epsilon$ and increase $z_{y'c'}$ by the same value. This move will keep $\Delta$ unchanged but decreases $\sum_{(y,c) \in \mathcal{T}_e} z_{yc}\cdot p(c)\cdot |y|$ which violates the optimality assumption.

For every optimizer $z_{y,c}$, let us define:
$$\beta_s(z_{y,c})=\max_{(c,y) \in \mathcal{T}_e: z_{yc} > 0} p(c)\cdot|y|.
$$
It follows from the above property and the definition of $\beta_s(z_{y,c})$ that
\begin{align*}
    & p(c).|y| > \beta_s(z_{y,c}) \implies z_{yc}=0,\\
& p(c).|y| < \beta_s(z_{y,c}) \implies z_{yc} = p(y|c),
\end{align*}
It remains to show that $\beta_s(z_{y,c})=\beta_s(z'_{y,c})$
for any two maximizers $z_{y,c}$ and $z'_{y,c}$.
Assume otherwise that 
$\beta_s(z'_{y,c})>\beta_s(z_{y,c})$.
Then, we obtain that  $$z_{yc}>0\implies
p(c).|y| \leq  \beta_s(z_{y,c})\implies
p(c).|y| <\beta_s(z'_{y,c})\implies z'_{yc}=p(y|c)\geq z_{yc}.$$
Moreover, $z'_{yc}\geq z_{yc}$ also holds when $z_{yc}=0$. Thus, $z'_{yc}\geq z_{yc}$ holds regardless of the value of $z_{yc}$.  
On the other hand
\[
\sum_{(y,c) \in \mathcal{T}_e} z'_{yc} = \sum_{(y,c) \in \mathcal{T}_e} z_{yc} = \Delnp
\]
Therefore, we must have 
$z'_{yc}=z_{yc}$ for all $(y,c)\in\mathcal{T}_e$, which is in contradiction with $\beta_s(z'_{y,c})>\beta_s(z_{y,c})$.
\end{proof}
\subsection{Proof of Theorem \ref{th:effective}}\label{appndxproofth:effective}

\begin{proof}[Proof of Theorem \ref{th:effective}] 
Take an arbitrary  
\[
(y^*, c^*) \in \argmin_{(y,c)\in\mathcal{T}_e} p(c)\cdot|y|.
\]
Using the assumption $\Delnp > 0$, if $\lambda \cdot g'_-(\Delnp) > \beta_e$ we conclude that there exists $0 < \epsilon < \min\{\beta_e,\, \Delnp,\, p(y^*|c^*)\}$ such that
\[
\lambda (g(\Delnp) - g(\Delnp - \epsilon)) > \epsilon\cdot\beta_e.
\]
Thus,
\begin{align}
\Uum - \epsilon\cdot\beta_e - \lambda g(\Delnp - \epsilon) > \Uum - \lambda g(\Delnp).\label{eqnuep}
\end{align}

Consider the representation in Lemma \ref{lemma-represent} and the optimization 
\begin{align}
    \maximize&\; \Uum - \sum_{(y,c) \in \mathcal{T}_e} z_{yc}\cdot p(c)\cdot |y| - \lambda \cdot g\left(\Delnp - \sum_{(y,c) \in \mathcal{T}_e} z_{yc}\right) \label{eq:prob-simprep}
\end{align}
where the maximum is over $z_{yc}$ satisfying \eqref{constd1} and \eqref{constd2}.
Let $z'_{yc}$ be defined by
\begin{align*}
    z'_{yc}= \begin{cases}
        \epsilon & \text{if }(y,c) = (y^*, c^*) \\
        0 & \text{otherwise.} \\
        \end{cases}
\end{align*} 
One can directly inspect that the variables $z'_{yc}$ satisfy the constraints \eqref{constd1} and \eqref{constd2}. Thus, the maximum value of \eqref{eq:prob-simprep} is greater than or equal to the expression evaluated at $z'_{yc}$. In other words, if $\tilde z_{yc}$ is any maximizer in \eqref{eq:prob-simprep}, we have:
\begin{align}
\Uum& - \sum_{(y,c) \in \mathcal{T}_e} \tilde z_{yc}\cdot p(c)\cdot |y| - \lambda \cdot g\left(\Delnp - \sum_{(y,c) \in \mathcal{T}_e} \tilde z_{yc}\right)\nonumber \\
\geq&~ \Uum - z'_{y^* c^*}\cdot \beta_e - \lambda \cdot g\left(\Delnp - z'_{y^* c^*}\right) \nonumber\\
>&~ \Uum - \lambda \cdot g(\Delnp)\label{eqntoexpeqnuep}
\end{align}
where \eqref{eqntoexpeqnuep} follows from \eqref{eqnuep}. 
This in turn implies $g(\Delnp - \sum_{(y,c) \in \mathcal{T}_e} \tilde z_{yc}) < g(\Delnp)$. Therefore, for any optimal solution, we have  $\Delta < \Delnp$, and the penalty is effective.

On the other hand, assume that $\lambda \cdot g'_-(\Delnp) \leq \beta_e$. We show that the utility-maximizing solution is optimal for the problem in \eqref{eq:prob-simprep}. We need to show that for any arbitrary $z_{yc}$, we have
\begin{align}
    \nonumber&\Uum - \lambda \cdot g(\Delnp)  \geq \Uum - \sum_{(y,c) \in \mathcal{T}_e} z_{yc}\cdot p(c)\cdot |y| - \lambda \cdot g\left(\Delnp - \sum_{(y,c) \in \mathcal{T}_e} z_{yc}\right).
    \end{align}
Note that
\begin{align}
    \nonumber&\left[\Uum - \lambda \cdot g(\Delnp) \right] - \left[\Uum - \sum_{(y,c) \in \mathcal{T}_e} z_{yc}\cdot p(c)\cdot |y| - \lambda \cdot g\left(\Delnp - \sum_{(y,c) \in \mathcal{T}_e} z_{yc}\right) \right]\\
    &=\nonumber \lambda \left[ g\left(\Delnp - \sum_{(y,c) \in \mathcal{T}_e} z_{yc}\right) -  g(\Delnp)\right]  + \sum_{(y,c) \in \mathcal{T}_e} z_{yc}\cdot p(c)\cdot |y|\\
     &\geq\nonumber \lambda \left[ g\left(\Delnp - \sum_{(y,c) \in \mathcal{T}_e} z_{yc}\right) -  g(\Delnp)\right]  + \beta_e \sum_{(y,c) \in \mathcal{T}_e} z_{yc}
     \\&\geq  - \lambda \cdot g'_-(\Delnp) \sum_{(y,c) \in \mathcal{T}_e} z_{yc}  + \beta_e \sum_{(y,c) \in \mathcal{T}_e} z_{yc}\label{eqnxnex}
     \\&= \left(\beta_e- \lambda \cdot g'_-(\Delnp)   \right)\sum_{(y,c) \in \mathcal{T}_e} z_{yc}\nonumber\\&
     \geq 0.\label{eqnxnex2}
\end{align}
where \eqref{eqnxnex} follows from the property $g(x) - g(y) \geq g'(y)(x-y)$ of convex functions, and \eqref{eqnxnex2} follows from  $\sum_{(y,c) \in \mathcal{T}_e} z_{yc} \geq 0$ and the assumption $\lambda \cdot g'_-(\Delnp) \leq \beta_e$. Thus, we obtain that the utility-maximizing solution is optimal in this case.

\end{proof}

\subsection{Proof of Theorem \ref{th:fulldp}}
\label{appndxproofth:fulldp}

\begin{proof}[Proof of Theorem \ref{th:fulldp}]

Assume that $\beta_s \leq \lambda\cdot g'_+(0)$. Take an arbitrary optimal solution to \eqref{eq:demo-parity-simp} and denote it by $z_{yc}$. Note that $z_{yc}$ is in the domain of the optimization problem in \eqref{eq:prob-simp}. We would like to show that $z_{yc}$ is also a maximizer for the problem in \eqref{eq:prob-simp}. This will prove that the penalty is fully satisfactory, meaning there is a policy satisfying demographic parity which is also optimal.

From the identification in \eqref{eq:substitute2}, the value $\beta_s$ introduced in the theorem can be stated as
$$\beta_s = \max_{(c,y) \in \mathcal{T}_e: z_{yc} > 0} p(c)\cdot|y| .
$$
Note that Lemma \ref{lemmanew4} shows that this quantity is the same for all maximizers of the \eqref{eq:demo-parity-simp}. Take any maximizer for \eqref{eq:prob-simp} denoted by $z'_{yc}$. 
We need to prove that the objective function of \eqref{eq:prob-simp} at $z_{yc}$ is not smaller than that at $z'_{yc}$, i.e.,
\begin{align*}
    &\Uum - \sum_{(y,c) \in \mathcal{T}_e} z_{yc}\cdot p(c)\cdot |y| - \lambda \cdot g(0)  \geq \Uum - \sum_{(y,c) \in \mathcal{T}_e} z'_{yc}\cdot p(c)\cdot |y| - \lambda \cdot g\left(\Delnp - \sum_{(y,c) \in \mathcal{T}_e} z'_{yc}\right) \end{align*}
This would show that $z_{yc}$ is a maximizer for \eqref{eq:prob-simp}.

We have:
\begin{align}
    \nonumber&\left[\Uum - \sum_{(y,c) \in \mathcal{T}_e} z_{yc}\cdot p(c)\cdot |y| - \lambda \cdot g(0) \right] - \left[\Uum - \sum_{(y,c) \in \mathcal{T}_e} z'_{yc}\cdot p(c)\cdot |y| - \lambda \cdot g\left(\Delnp - \sum_{(y,c) \in \mathcal{T}_e} z'_{yc}\right) \right]\\\nonumber
    =& \left[\sum_{(y,c) \in \mathcal{T}_e} z'_{yc}\cdot p(c)\cdot |y| - \sum_{(y,c) \in \mathcal{T}_e} z_{yc}\cdot p(c)\cdot |y| \right] + \lambda \left[ g\left(\Delnp - \sum_{(y,c) \in \mathcal{T}_e} z'_{yc}\right) - g(0) \right]\\
     \geq& \left[\sum_{(y,c) \in \mathcal{T}_e} z'_{yc} - \sum_{(y,c) \in \mathcal{T}_e} z_{yc} \right] \beta_s + \lambda \left[g\left(\Delnp - \sum_{(y,c) \in \mathcal{T}_e} z'_{yc}\right) - g(0)\right] \label{eqnfdr1}\\
    \geq& \left[\sum_{(y,c) \in \mathcal{T}_e} z'_{yc} - \sum_{(y,c) \in \mathcal{T}_e} z_{yc} \right] \beta_s + \lambda  \left[\Delnp - \sum_{(y,c) \in \mathcal{T}_e} z'_{yc} \right] g'_+(0)\label{eqnfdr2}
    \\
    =&  \left[\Delnp - \sum_{(y,c) \in \mathcal{T}_e} z'_{yc} \right] (\lambda\cdot g'_+(0)-\beta_s)\nonumber
    \\\geq& ~0.\label{eqnfdr3}
\end{align}
The inequality \eqref{eqnfdr2} follows from the convexity of $g(\cdot)$ and \eqref{eqnfdr3} follows from 
$\beta_s \leq \lambda\cdot g'_+(0)$ and $\Delnp - \sum_{(y,c) \in \mathcal{T}_e} z'_{yc} \geq 0$. The inequality \eqref{eqnfdr1} follows from 
$$(z'_{yc} - z_{yc})\cdot p(c)\cdot |y| \geq (z'_{yc} - z_{yc})\cdot \beta_s, \qquad\forall (c,y)\in\mathcal{T}_e$$
which can be argued as follows:  Lemma \ref{lemmanew4} shows that
\begin{align*}
    & p(c).|y| > \beta_s \implies z_{yc}=0,\\
& p(c).|y| < \beta_s \implies z_{yc} = p(y|c).
\end{align*}
Take some $(c,y)\in\mathcal{T}_e$. If $p(c)\cdot |y|>\beta_s$, we have $z_{yc}=0$ and hence
$$z'_{yc}\cdot p(c)\cdot |y| - z_{yc}\cdot p(c)\cdot |y|\geq \left[z'_{yc} - z_{yc} \right] \beta_s$$
Next, if $p(c)\cdot |y|=\beta_s$, we have
$$z'_{yc}\cdot p(c)\cdot |y| - z_{yc}\cdot p(c)\cdot |y|=\left[z'_{yc} - z_{yc} \right] \beta_s.$$
Finally, if $p(c)\cdot |y|<\beta_s$, we have $z_{yc}=p(y|c)$ and hence $z'_{yc}-z_{yc}\leq 0$ which implies that
\[
 (z'_{yc} - z_{yc})\cdot p(c)\cdot |y| \geq (z'_{yc} - z_{yc})\cdot \beta_s.
\]
This completes the proof for one direction.

Next, assume that $\beta_s > \lambda\cdot g'_+(0)$. Take an optimal solution $z_{yc}$ to the problem \eqref{eq:demo-parity-simp}. Let 
\[
(c^*,y^*) \in \argmax_{y,c} \{p(c)\cdot|y| : 
(c,y) \in \mathcal{T}_e \land  z_{yc} > 0) \} 
\]
Find $\epsilon>0$ such that $z_{y^*c^*} > \epsilon$ and
\begin{align}
\beta_s \cdot \epsilon > \lambda (g(\epsilon) - g(0)).\label{eqnepla}
\end{align}
Define 
\begin{align*}
    z'_{yc}= \begin{cases}
        z_{yc} - \epsilon & \text{if }(y,c) = (y^*, c^*) \\
        z_{yc} & \text{otherwise.} \\
        \end{cases}
\end{align*} 
The assignment $z'_{yc}$ satisfies \eqref{constd1} and \eqref{constd2} and belongs to the domain of the optimization problem in \eqref{eq:prob-simp}.
Note that \eqref{eqnepla} implies that
$$-z'_{y^* c^*}\cdot p(c^*)\cdot |y^*| - \lambda \cdot g(\epsilon)  > -z_{y^* c^*}\cdot p(c^*)\cdot |y^*| - \lambda \cdot g(0).$$
Therefore,
\begin{align*}
    &\Uum - \sum_{(y,c) \in \mathcal{T}_e} z'_{yc}\cdot p(c)\cdot |y| - \lambda \cdot g\left(\Delnp - \sum_{(y,c) \in \mathcal{T}_e} z'_{yc}\right)  \\&> \Uum - \sum_{(y,c) \in \mathcal{T}_e} z_{yc}\cdot p(c)\cdot |y| - \lambda \cdot g\left(\Delnp - \sum_{(y,c) \in \mathcal{T}_e} z_{yc}\right).
\end{align*}
Thus, $z_{yc}$ is not a maximizer for \eqref{eq:prob-simp} and the penalty is not fully satisfactory.

\end{proof}

\subsection{Proof of Theorem \ref{th:um}}\label{appndxproofth:util-max}

\begin{proof}[Proof of Theorem \ref{th:um}]
    Note that when $\lambda =0$ the optimal decision at time $t$ will be to select an individual if and only if $Y_t>0$ (remember the assumption that $0\notin\mathcal{Y}$, i.e., hiring each individual is either beneficial or harmful).   
    \footnote{Note that the optimal decision in this case only depends on $Y_t$ and not the distribution of $p(c,y_t)$ (as in the general case of $\lambda>0$).}  So for each $c \in \G$ we can write
    \begin{align*}
        \forall y\in\mathcal{Y}: \; \Pqt{t+1}{y}{c} 
        &= \sum_{y'\in\mathcal{Y}} \sum_{d\in\{0,1\}} \Pqt{t}{y'}{c} \cdot \Pr(D=d|Y=y',C=c) \cdot q(y|y',d) \\
        &= \sum_{y'<0} \Pqt{t}{y'}{c} \cdot q(y|y',0) + \sum_{y'>0} \Pqt{t}{y'}{c} \cdot q(y|y',1).
    \end{align*}
    So, if we define
    \[
        q^*(y|y') = \begin{cases}
        q(y|y',1)&\text{if } y'>0, \\
        q(y|y',0)& \text{if }y'<0, \\
        \end{cases}
    \] 
    we have
    \[
    \forall y\in\mathcal{Y}: \; \Pqt{t+1}{y}{c} = \sum_{y'} \Pqt{t}{y'}{c} \cdot q^*(y|y').
    \]
    Therefore, the updates of qualifications within each group form a discrete-time, finite support, time-homogenous Markov chain in which their transition probabilities are identical, given by $\Pr(Y_{t+1}=y|Y_t=y') = q^*(y|y')$. The assumption $q(y|y',d) >0$ for all $y,y'$ and $d$ implies $q^*(y|y') > 0$ for all $y$ and $y'$. This, in turn, implies that these Markov chains are aperiodic and irreducible and have an identical, unique, absorbing stationary distribution. So, regardless of the initial distributions $p_{Y_0|A}$ and $p_{Y_0|B}$, we have
    \[
    \lim_{t \rightarrow \infty} \|p_{Y_t|A}-p_{Y_t|B}\|_1=0.
    \]
    Next, for the case of $\lambda > 0$, consider Example \ref{example1}. In this example, we consider $\mathcal{Y} = \{-2, -1, 2\}$ and update probabilities presented in Figure \ref{fig:tran}. Assume that the groups $\A$ and $\B$ are of equal size $p(C=\A) = p(C=\B)=0.5$ and the discrimination penalty is $g(x)=x$. Some plots are given in Example \ref{example1} in Figure \ref{fig:plot_discrimination}. Here, we provide further calculations that explain the behavior described in  Figure \ref{fig:plot_discrimination} for $\lambda = 0.7$.
    
    If we index the elements of $\mathcal{Y}$ in increasing order, a stationary state of the system is as follows:
    \begin{align}
        p_{Y|\A} = \left[\frac{1}{3}, \frac{1}{10}, \frac{17}{30}\right] \nonumber \\
        p_{Y|\B} = \left[\frac{17}{30}, \frac{1}{10}, \frac{1}{3}\right]. 
    \label{eq:state-disparity}
    \end{align}
    In this stationary state, the institution's selection policy will be
    \begin{align}
        p_{D=1|Y_t,\A} = [0, 0, 1], \label{eqnspT1}\\
        p_{D=1|Y_t,\B} = [0, 1, 1], \label{eqnspT2}
    \end{align}
    which will result in $\Delta_t = \left|\frac{17}{30} - \frac{1}{10} - \frac{1}{3} \right| = \frac{4}{30}\approx 0.133...$. 
    Now, consider the initial state as
    \begin{align*}
        p_{Y_0|\A} = [0.3, 0.1, 0.6],\\
        p_{Y_0|\B} = [0.5, 0.1, 0.4].
    \end{align*}
    In this case, the initial selection disparity equals $\Delta_0 = 0.1<\frac{4}{30}$.
    The selection policy, in this case, is again the one given in \eqref{eqnspT1} and \eqref{eqnspT2} for all $t>0$  and the qualification distributions converge to the stationary state 
    \eqref{eq:state-disparity}. Therefore, 
$\lim_{t\rightarrow\infty}\Delta_t=\frac{4}{30}>\Delta_0 = 0.1$ meaning that
    the selection disparity will increase over time. This is illustrated in Figure \ref{fig:plot_discrimination} for the choice of $\lambda=0.7$.
\end{proof}

\subsection{Proof of Theorem \ref{th:penalty-conv}}\label{appndxproofth:penalty-conv}

\begin{proof}[Proof of Theorem \ref{th:penalty-conv}] 
    Fix some time instance $t$. Without loss of generality, assume \begin{align}
        \Pr(Y_t>0|C=\A) \geq \Pr(Y_t>0|C=\B)\label{eqnnewassumption}.
    \end{align} Otherwise, if the above assumption does not hold, we can just swap the names of the two groups. 
    
    Define $\alpha = \min_{y,y',d} \{q(y|y',d)\}$ and 
    \[f(y',d) =
        \Pqt{t}{y'}{\A}\cdot \Pr(D_t=d|Y_t=y',C=\A) - \Pqt{t}{y'}{\B}\cdot \Pr(D_t=d|Y_t=y',C=\B).
        \]
    First, using simple algebraic (in)equalities we show that
    \begin{align}
\sum_{y\in\mathcal{Y}} | \Pqt{t+1}{y}{\A} - \Pqt{t+1}{y}{\B}| \leq (1 - |\mathcal{Y}| \alpha) \sum_{y'\in\mathcal{Y}} \sum_{d\in\{0,1\}}
          \left|f(y',d)\right|.\label{eqneqnn1}
      \end{align}
    Then, by utilizing Lemma \ref{lem1} in the institution's policy at time $t$ we deduce
    \begin{align}
    \sum_{y'\in\mathcal{Y}}
          \left|f(y',1)\right| \leq \sum_{y'} \left| \Pqt{t}{y'}{\A}- \Pqt{t}{y'}{\B} \right|\label{eqnsecpart}
    \end{align}
    and similarly
    \begin{align}
    \sum_{y'\in\mathcal{Y}}
          \left|f(y',0)\right| \leq \sum_{y'} \left| \Pqt{t}{y'}{\A}- \Pqt{t}{y'}{\B} \right|.\label{eqnsecpart2}
    \end{align}
    Combining them we get the desired result
    \begin{gather*}
        \sum_{y\in\mathcal{Y}} | \Pqt{t+1}{y}{\A} - \Pqt{t+1}{y}{\B}| \leq (1 - |\mathcal{Y}| \alpha) \sum_{y'\in\mathcal{Y}} \sum_{d\in\{0,1\}} 
          \left|f(y',d)\right| \leq 2 (1 - |\mathcal{Y}| \alpha) \sum_{y\in\mathcal{Y}} | \Pqt{t}{y}{\A} - \Pqt{t}{y}{\B}|.
    \end{gather*}
    Consequently, if   $2 (1 - |\mathcal{Y}| \alpha) < 1$, then
    \begin{gather*}
        \lim_{t \rightarrow \infty} \sum_{y\in\mathcal{Y}} | \Pqt{t}{y}{\A} - \Pqt{t}{y}{\B}| = 0.
    \end{gather*}

    To show \eqref{eqneqnn1}, let us denote
    $$q^*(y|y',d) = \frac{q(y|y',d) - \alpha}{1 - |\mathcal{Y}|\alpha}.$$
    Then, we can write:
    \begin{align}
        &\sum_{y\in\mathcal{Y}} | \Pqt{t+1}{y}{\A} - \Pqt{t+1}{y}{\B}| \nonumber\\
        =& \sum_{y\in\mathcal{Y}} \left| \sum_{y'\in\mathcal{Y}} \sum_{d\in\{0,1\}}
        \left[
        \Pqt{t}{y'}{\A}\cdot \Pr(D_t=d|Y_t=y',C=\A) \cdot q(y|y',d) - \Pqt{t}{y'}{\B}\cdot \Pr(D_t=d|Y_t=y',C=\B) \cdot q(y|y',d) \right] \right| \nonumber\\
         =& \sum_{y\in\mathcal{Y}} \left| \sum_{y'\in\mathcal{Y}} \sum_{d\in\{0,1\}}
         q(y|y',d) f(y',d)\right| \nonumber\\
         =& \sum_{y\in\mathcal{Y}} \left| \sum_{y'\in\mathcal{Y}} \sum_{d\in\{0,1\}}
         (1 - |\mathcal{Y}| \alpha)q^*(y|y',d) f(y',d)\right| \label{eqneq2}\\
         \leq& (1 - |\mathcal{Y}| \alpha) \sum_{y\in\mathcal{Y}} \sum_{y'\in\mathcal{Y}} \sum_{d\in\{0,1\}}
         q^*(y|y',d) \left|f(y',d)\right| \nonumber\\
         =& (1 - |\mathcal{Y}| \alpha) \sum_{y'\in\mathcal{Y}} \sum_{d\in\{0,1\}}
          \left|f(y',d)\right| \sum_{y\in\mathcal{Y}} q^*(y|y',d) \nonumber\\
         =& (1 - |\mathcal{Y}| \alpha) \sum_{y'\in\mathcal{Y}} \sum_{d\in\{0,1\}}
          \left|f(y',d)\right|
          \label{pr:penalty-conv2}
    \end{align}
    where \eqref{eqneq2} follows from 
    \begin{gather*}
\sum_{y'\in\mathcal{Y}} \sum_{d\in\{0,1\}} f(y',d) = \sum_{y'\in\mathcal{Y}} 
        \left[
        \Pqt{t}{y'}{\A} - \Pqt{t}{y'}{\B} \right] = 0
    \end{gather*}
    and \eqref{pr:penalty-conv2} follows from 
    \begin{gather*}
        \sum_{y\in\mathcal{Y}} q^*(y|y',d) = \sum_{y\in\mathcal{Y}}  \frac{q(y|y',d) - \alpha}{1 - |\mathcal{Y}|\alpha} = \frac{1}{1 - |\mathcal{Y}|\alpha} - \frac{|\mathcal{Y}|\alpha}{1 - |\mathcal{Y}|\alpha} = 1.
    \end{gather*}
Next, we show \eqref{eqnsecpart}. Define 
$$y_\A = \max_y \{y:~\forall y' < y\quad p_t(Y_t=y'|C=\A)>0 \implies \Sct{Y_t=y', C=\A} = 0\}$$
for the decision rule of the institution at time $t$. The 4th statement in Lemma \ref{lem1} along with the assumption that $0\notin\mathcal{Y}$ implies $y_\A > 0$. Furthermore, the lemma yields
    ~
    \begin{align*}
        \;p_t(y_\A|\B) \cdot \Sct{Y_t=y_\A, C=\B}& = p_t(y_\A|\B), \\
        \forall y > y_\A, c \in \{\A,\B\}: \; p_t(y|c) \cdot \Sct{Y_t=y, C=c} &= p_t(y|c), \\
        \forall y < y_\A: \; p_t(y|\A) \Sct{Y_t=y, C=\A} &= 0. 
    \end{align*}
    ~
    Using these facts we can write
    ~
    \begin{align}
        \sum_{y'\in\mathcal{Y}}
          \left|f(y',1)\right| =& 
          \sum_{y'\in\mathcal{Y}} \left| \Pqt{t}{y'}{\A}\cdot \Sct{Y_t=y',C=\A} - \Pqt{t}{y'}{\B}\cdot \Sct{Y_t=y',C=\B} \right| \nonumber \\
          =& \sum_{y' > y_\A} \left| \Pqt{t}{y'}{\A}- \Pqt{t}{y'}{\B}\right| \nonumber \\
          &+ \sum_{y' < y_\A} \Pqt{t}{y'}{\B} \cdot\Sct{Y_t=y',C=\B} \nonumber \\
          &+ \left| \Pqt{t}{y_\A}{\A}\cdot \Sct{Y_t=y_\A,C=\A} - \Pqt{t}{y_\A}{\B} \right| \nonumber \\
          \leq& \sum_{y' > y_\A} \left| \Pqt{t}{y'}{\A}- \Pqt{t}{y'}{\B}\right| \nonumber \\
                    &+ \Pqt{t}{y_\A}{\B} \cdot \Sctinv{Y_t=y_\A,C=\A} \nonumber\\
          &+ \sum_{y' < y_\A} \Pqt{t}{y'}{\B} \cdot\Sct{Y_t=y',C=\B} \nonumber \\
          &+ \left| \Pqt{t}{y_\A}{\A} - \Pqt{t}{y_\A}{\B} \right| \cdot \Sct{Y_t=y_\A,C=\A} 
\label{pr:penalty-conv1}
    \end{align}
where \eqref{pr:penalty-conv1} follows from the triangle's inequality.

    On the other hand, part 3 of Lemma \ref{lem1} along with the assumption in \eqref{eqnnewassumption} shows that 
    \begin{gather*}
        \sum_{y'\in\mathcal{Y}} \Pqt{t}{y'}{\A}\cdot \Sct{Y_t=y',C=\A} \geq \sum_{y'\in\mathcal{Y}} \Pqt{t}{y'}{\B}\cdot \Sct{Y_t=y',C=\B}.
    \end{gather*}    
    We can expand it as
    \begin{align*}
 &
        \Pqt{t}{y_\A}{\B} +
        \sum_{y' < y_\A } \Pqt{t}{y'}{\B} \cdot\Sct{Y_t=y',C=\B}  
        + \sum_{y' > y_\A}  \Pqt{t}{y'}{\B} 
        \\
         &\leq\Pqt{t}{y_\A}{\A} \cdot \Sct{Y_t=y_\A,C=\A} +
        \sum_{y' > y_\A}  \Pqt{t}{y'}{\A}
    \end{align*}
    which implies
    \begin{align}
       & \Pqt{t}{y_\A}{\B} \cdot \Sctinv{Y_t=y_\A,C=\A} +\sum_{y' < y_\A} \Pqt{t}{y'}{\B} \cdot\Sct{Y_t=y',C=\B}\nonumber\\
         &\leq  \left( \Pqt{t}{y_\A}{\A} - \Pqt{t}{y_\A}{\B} \right) \cdot \Sct{Y_t=y_\A,C=\A}  
        + \sum_{y' > y_\A} \left[ \Pqt{t}{y'}{\A}- \Pqt{t}{y'}{\B} \right]\nonumber
        \\&=\left( \Pqt{t}{y_\A}{\B} - \Pqt{t}{y_\A}{\A} \right) \cdot \Sctinv{Y_t=y_\A,C=\A} + \sum_{y' < y_\A} \left( \Pqt{t}{y'}{\B}- \Pqt{t}{y'}{\A} \right)\label{eqn16p1}
        \\&
           \leq \left| \Pqt{t}{y_\A}{\B} - \Pqt{t}{y_\A}{\A} \right| \cdot \Sctinv{Y_t=y_\A,C=\A} + \sum_{y' < y_\A} \left| \Pqt{t}{y'}{\B}- \Pqt{t}{y'}{\A} \right|\label{eqn16p2}
    \end{align}
    where \eqref{eqn16p1} follows from the fact that for $c \in \G$ we have
    \[
     \Pqt{t}{y_\A}{c} \cdot \Sct{Y_t=y_\A,C=c} + \sum_{y' > y_\A} \Pqt{t}{y'}{c} = 1 - \Pqt{t}{y_\A}{c} \cdot \Sctinv{Y_t=y_\A,C=c} - \sum_{y' < y_\A} \Pqt{t}{y'}{c}.
    \]
    From \eqref{eqn16p2} and \ref{pr:penalty-conv1} we get
\begin{align*}
           \sum_{y'\in\mathcal{Y}}
          \left|f(y',1)\right| 
          &\leq  \sum_{y' > y_\A} \left| \Pqt{t}{y'}{\A}- \Pqt{t}{y'}{\B}\right| \nonumber \\
                    &\qquad+\left| \Pqt{t}{y_\A}{\B} - \Pqt{t}{y_\A}{\A} \right| \cdot \Sctinv{Y_t=y_\A,C=\A} + \sum_{y' < y_\A} \left| \Pqt{t}{y'}{\B}- \Pqt{t}{y'}{\A} \right|\\
          &\qquad+ \left| \Pqt{t}{y_\A}{\A} - \Pqt{t}{y_\A}{\B} \right| \cdot \Sct{Y_t=y_\A,C=\A}
 \\
            &= \sum_{y'} \left| \Pqt{t}{y'}{\A}- \Pqt{t}{y'}{\B} \right|.
    \end{align*}
This completes the proof for \eqref{eqnsecpart}.

    It remains to prove \eqref{eqnsecpart2}, i.e., to show that
    \[
    \sum_{y'\in\mathcal{Y}}
          \left|f(y',0)\right| \leq \sum_{y'} \left| \Pqt{t}{y'}{\A}- \Pqt{t}{y'}{\B} \right|.
    \]
    We follow a similar approach. Define 
    $$y_\B = \min_y \{y:~\forall y' > y\quad p_t(Y_t=y'|C=\B)>0 \implies \Sct{Y_t=y', C=\B} = 1\}.$$
    The 4th statement in Lemma \ref{lem1} along with the assumption that $0\notin\mathcal{Y}$ implies $y_\B < 0$. Furthermore, the lemma yields
    ~
    \begin{align*}
        \;p_t(y_\B|\A) \cdot \Sctinv{Y_t=y_\B, C=\A}& = p_t(y_\B|\A), \\
        \forall y < y_\B, c \in \{\A,\B\}: \; p_t(y|c) \cdot \Sctinv{Y_t=y, C=c} &= p_t(y|c), \\
        \forall y > y_\B: \; p_t(y|\B) \cdot \Sctinv{Y_t=y, C=\B} &= 0. 
    \end{align*}

    Now, these imply:
    
    \begin{align}
        \sum_{y'\in\mathcal{Y}}
          \left|f(y',0)\right| =& 
          \sum_{y'\in\mathcal{Y}} \left| \Pqt{t}{y'}{\A}\cdot \Sctinv{Y_t=y',C=\A} - \Pqt{t}{y'}{\B}\cdot \Sctinv{Y_t=y',C=\B} \right| \nonumber \\
          =& \sum_{y' < y_\B} \left| \Pqt{t}{y'}{\A}- \Pqt{t}{y'}{\B}\right| \nonumber \\
          &+ \sum_{y' > y_\B} \Pqt{t}{y'}{\A} \cdot\Sctinv{Y_t=y',C=\A} \nonumber \\
          &+ \left| \Pqt{t}{y_\B}{\B}\cdot \Sctinv{Y_t=y_\B,C=\B} - \Pqt{t}{y_\B}{\A} \right| \nonumber \\
          \leq& \sum_{y' < y_\B} \left| \Pqt{t}{y'}{\A}- \Pqt{t}{y'}{\B}\right| \nonumber \\
                    &+ \Pqt{t}{y_\B}{\A} \cdot \Sct{Y_t=y_\B,C=\B} \nonumber\\
          &+ \sum_{y' > y_\B} \Pqt{t}{y'}{\A} \cdot\Sctinv{Y_t=y',C=\A} \nonumber \\
          &+ \left| \Pqt{t}{y_\B}{\B} - \Pqt{t}{y_\B}{\A} \right| \cdot \Sctinv{Y_t=y_\B,C=\B}.\label{pr:penalty-conv1v}
    \end{align}

On the other hand, part 3 of Lemma \ref{lem1} along with the assumption in \eqref{eqnnewassumption} shows that 
    \begin{gather*}
        \sum_{y'\in\mathcal{Y}} \Pqt{t}{y'}{\B}\cdot \Sctinv{Y_t=y',C=\B} \geq \sum_{y'\in\mathcal{Y}} \Pqt{t}{y'}{\A}\cdot \Sctinv{Y_t=y',C=\A},
    \end{gather*}   
    We can expand it as
    \begin{align*}
 &
        \Pqt{t}{y_\B}{\A} +
        \sum_{y' > y_\B } \Pqt{t}{y'}{\A} \cdot\Sctinv{Y_t=y',C=\A}  
        + \sum_{y' < y_\B}  \Pqt{t}{y'}{\A} 
        \\
         &\leq\Pqt{t}{y_\B}{\B} \cdot \Sctinv{Y_t=y_\B,C=\B} +
        \sum_{y' < y_\B}  \Pqt{t}{y'}{\B}
    \end{align*}
    which implies
    \begin{align}
       & \Pqt{t}{y_\B}{\A} \cdot \Sct{Y_t=y_\B,C=\B} +\sum_{y' > y_\B} \Pqt{t}{y'}{\A} \cdot\Sctinv{Y_t=y',C=\A}\nonumber\\
         &\leq  \left( \Pqt{t}{y_\B}{\B} - \Pqt{t}{y_\B}{\A} \right) \cdot \Sctinv{Y_t=y_\B,C=\B}  
        + \sum_{y' < y_\B} \left[ \Pqt{t}{y'}{\B}- \Pqt{t}{y'}{\A} \right]\nonumber
        \\&=\left( \Pqt{t}{y_\B}{\A} - \Pqt{t}{y_\B}{\B} \right) \cdot \Sct{Y_t=y_\B,C=\B} + \sum_{y' > y_\B} \left( \Pqt{t}{y'}{\A}- \Pqt{t}{y'}{\B} \right)\label{eqn16p1v}
        \\&
           \leq \left| \Pqt{t}{y_\B}{\A} - \Pqt{t}{y_\B}{\B} \right| \cdot \Sct{Y_t=y_\B,C=\B} + \sum_{y' > y_\B} \left| \Pqt{t}{y'}{\A}- \Pqt{t}{y'}{\B} \right|\label{eqn16p2v}
    \end{align}
    where \eqref{eqn16p1v} follows from the fact that for $c \in \G$ we have
    \[
     \Pqt{t}{y_\B}{c} \cdot \Sctinv{Y_t=y_\B,C=c} + \sum_{y' < y_\B} \Pqt{t}{y'}{c} = 1 - \Pqt{t}{y_\B}{c} \cdot \Sct{Y_t=y_\B,C=c} - \sum_{y' > y_\B} \Pqt{t}{y'}{c}.
    \]
    From \eqref{eqn16p2v} and \ref{pr:penalty-conv1v} we get
\begin{align*}
           \sum_{y'\in\mathcal{Y}}
          \left|f(y',0)\right| 
          &\leq  \sum_{y' < y_\B} \left| \Pqt{t}{y'}{\A}- \Pqt{t}{y'}{\B}\right| \nonumber \\
                    &\qquad+\left| \Pqt{t}{y_\B}{\A} - \Pqt{t}{y_\B}{\B} \right| \cdot \Sct{Y_t=y_\B,C=\B} + \sum_{y' > y_\B} \left| \Pqt{t}{y'}{\A}- \Pqt{t}{y'}{\B} \right|\\
          &\qquad+ \left| \Pqt{t}{y_\B}{\B} - \Pqt{t}{y_\B}{\A} \right| \cdot \Sctinv{Y_t=y_\B,C=\A}
 \\
            &= \sum_{y'} \left| \Pqt{t}{y'}{\A}- \Pqt{t}{y'}{\B} \right|.
    \end{align*}

    This concludes the proof of \eqref{eqnsecpart2}.
        
\end{proof}

\subsection{Proof of Theorem \ref{thmnw5}} 
\label{appndxproofth:thmnw5}

    Note that
\begin{align*}
    \mathbb{E}[D_tY_{t+1}]&=
\mathbb{E}_{D_t,Y_t}[\mathbb{E}[D_tY_{t+1}|D_t,Y_t]]
\\&=p(D_t=1)
\mathbb{E}_{Y_t|D_t=1}[\mathbb{E}[Y_{t+1}|D_t=1,Y_t]]
\\&\geq p(D_t=1)
\mathbb{E}_{Y_t|D_t=1}[Y_t]
\\&=\mathbb{E}[D_tY_{t}].
\end{align*}
Thus, 
$\mathbb{E}[D_tY_{t}]\leq \mathbb{E}[D_tY_{t+1}]
$. Now, for step $t+1$, let us consider the specific policy 
\begin{align}p(D_{t+1}=d|c,Y_{t+1}=y_{t+1})\triangleq p(D_t=d|c,Y_{t+1}=y_{t+1}).\label{eqnspecificpolicy}\end{align}
Then, observe that 
$$p(D_{t+1}=d,C=c)=p(D_t=d,C=c)$$
and therefore $\Delta_t=\Delta_{t+1}$. Moreover,
$$\mathbb{E}[D_tY_{t}]\leq \mathbb{E}[D_{t+1}Y_{t+1}]
$$
This implies that if we use the policy \eqref{eqnspecificpolicy}, the objective function at time $t+1$, $ \E[D_{t+1}Y_{t+1}] - \lambda \cdot g(\F_{t+1})$ will be greater than or equal to the objective function at time $t$. Consequently if at time $t+1$, we maximize over all possible $p(d_{t+1}|c,y_{t+1})$, the maximum value of the objective function will also be higher than that at time $t$. In other words,
\begin{gather*}
\maximize_{p(d_t|c,y_t)}\; \E[D_tY_t] - \lambda \cdot g(\F_t)
\end{gather*}
will be a non-decreasing function of $t$.

\subsection{Proof of Lemma \ref{lem:stationary} and Theorem \ref{th:band-matrix}} 
\label{appndxproofth:band-matrix}

\begin{proof}[Proof of Lemma \ref{lem:stationary}] 
Assume that the distribution of $p(c,y_t)$ is such that 
\begin{align}
p(Y_t=y|C=\A)=p(Y_t=y|C=\B), \qquad\forall y\in\mathcal{Y}\label{eqnstl1}.
\end{align}
In particular, this implies that
$\Pr(Y_t>0|C=\A) =\Pr(Y_t>0|C=\B)$ and there is no tension between maximizing $\E[D_tY_t]$ and minimizing $\Delta_t$. In this case, the optimal policy is the utility-maximizing policy and $D_t=\mathbf{1}[Y_t>0]$ (see the beginning of Section \ref{sec:effstat} for a discussion of the utility-maximizing policy).
As discussed in proof of Theorem \ref{th:penalty-conv}, the evolution of qualifications within each group follows the same Markov chain with the transition matrix $T$ given in the statement of the lemma. Consequently, \eqref{eqnstl1} implies that
\begin{align}
p(Y_{t+1}=y|C=\A)=p(Y_{t+1}=y|C=\B), \qquad\forall y\in\mathcal{Y}.
\end{align}

If we take a stationary distribution $r_Y$ of the Markov chain $T$ and assume that the initial qualifications within each group have distribution $r_Y$:
\begin{align}
p(Y_0=y|C=\A)=p(Y_0=y|C=\B)=r_Y(y),
\end{align}
we get a stationary state of our dynamic, i.e., we will have
\begin{align}
p(Y_{t}=y|C=\A)=p(Y_{t}=y|C=\B)=r_Y(y), \qquad\forall t\geq 0.
\end{align}

\end{proof}

\begin{proof}[Proof of Theorem \ref{th:band-matrix}]
First, note that in a stationary state, if $\Pr(Y>0|\A) = \Pr(Y>0|\B)$, then, under any discrimination penalty, the decision policy for both groups is the same: selecting all positively qualified individuals. Therefore, the transition probabilities of the groups and the stationary state will be the ones described in Lemma \ref{lem:stationary}.

So, assume there exists a stationary state with societal disparity. Without loss of generality, assume that $\A$ is the advantaged group in the stationary state: 
\begin{align}
    \Pr(Y>0|\A) > \Pr(Y>0|\B).\label{eqnABatu}
\end{align}
Denote the qualification distribution of the groups in this stationary state by 
$$\pi_\A(y)=p(Y=y|C=\A)$$
and 
$$\pi_\B(y)=p(Y=y|C=\B)$$
and the corresponding decision policy $p^*(d|y,c)$.
We denote the transition matrices in this state for groups $\A$ and $\B$ by $T_\A$ and $T_\B$ respectively, which can be computed as follows:
\begin{align*}
    T_\A(i,j) &= q(y^{(j)}|y^{(i)}, 1) \cdot p^*(D=1|Y=y^{(i)}, C=\A) + q(y^{(j)}|y^{(i)}, 0) \cdot p^*(D=0|Y=y^{(i)}, C=\A), \\
    T_\B(i,j) &= q(y^{(j)}|y^{(i)}, 1) \cdot p^*(D=1|Y=y^{(i)}, C=\B) + q(y^{(j)}|y^{(i)}, 0) \cdot p^*(D=0|Y=y^{(i)}, C=\B). 
\end{align*}
Note that the assumption \eqref{assumptionqt} implies that $T_\A$ and $T_\B$ are aperiodic and irreducible. 

The last part of Lemma \ref{lem1} and the assumptions \eqref{eqnaa71} and \eqref{eqnaa72} imply that 
\begin{align*}
    y^{(i)} < 0 , i < n:& \; T_\A(i,i+1) = q(y^{(i+1)}|y^{(i)}, 0)  \\
    &\leq \; q(y^{(i+1)}|y^{(i)}, 1) \cdot p^*(D=1|Y=y^{(i)}, C=\B) + q(y^{(i+1)}|y^{(i)}, 0) \cdot p^*(D=0|Y=y^{(i)}, C=\B) 
    \\&= T_\B(i,i+1), \\
    y^{(i)} < 0 , i > 0:&T_\A(i,i-1) = q(y^{(i-1)}|y^{(i)}, 0)  \\
    &\geq \; q(y^{(i-1)}|y^{(i)}, 1) \cdot p^*(D=1|Y=y^{(i)}, C=\B) + q(y^{(i-1)}|y^{(i)}, 0) \cdot p^*(D=0|Y=y^{(i)}, C=\B) 
    \\&= T_\B(i,i-1), \\
    y^{(i)} > 0, i < n:& \; T_\A(i, i+1) = q(y^{(i+1)}|y^{(i)}, 1) \cdot p^*(D=1|Y=y^{(i)}, C=\A) + q(y^{(i+1)}|y^{(i)}, 0) \cdot p^*(D=0|Y=y^{(i)}, C=\A) \\
    &\leq \; q(y^{(i+1)}|y^{(i)}, 1) 
    \\&= T_\B(i, i+1), \\
    y^{(i)} > 0, i > 0:& \; T_\A(i, i-1) = q(y^{(i-1)}|y^{(i)}, 1) \cdot p^*(D=1|Y=y^{(i)}, C=\A) + q(y^{(i-1)}|y^{(i)}, 0) \cdot p^*(D=0|Y=y^{(i)}, C=\A) \\
    &\geq \; q(y^{(i-1)}|y^{(i)}, 1) \\&= T_\B(i, i-1).
\end{align*}
Therefore, Lemma \ref{lem:adv} (given below) yields
\[
\sum_{y^{(i)} > 0} \pi_\A(y^{(i)}) \leq \sum_{y^{(i)} > 0} \pi_\B(y^{(i)}),
\]
or $\Pr(Y>0|\A) \leq \Pr(Y>0|\B)$. But this contradicts the assumption in \eqref{eqnABatu} that  $\A$ is the advantaged group.
\end{proof}

\begin{lemma} \label{lem:adv}
Consider two (discrete-time) Markov process $\{A_t\}$ and  $\{B_t\}$ sharing the state support $\mathcal{Y} = \{y^{(1)}, \cdots, y^{(n)}\}$ such that $y^{(1)} \leq \cdots \leq y^{(n)}$. Assume these two processes are identified by transition matrices $P$ and $Q$ such that for $i,j\in\{1,2,\cdots,n\}$ we have
$$P_{i,j}=\Pr(A_t=y^{(j)}|A_{t-1}=y^{(i)}),$$
$$Q_{i,j}=\Pr(B_t=y^{(j)}|B_{t-1}=y^{(i)}).$$
Assume that $A_t$ and $B_t$ are birth-death Markov processes and their transition matrices satisfy
\begin{align*}
        |i-j| \leq 1 \iff P_{i,j} > 0, \\
        |i-j| \leq 1 \iff Q_{i,j} > 0. 
\end{align*}
Assume $P$ and $Q$ satisfy the following for every $1\leq k\leq n-1$:
\begin{align}
        &P_{k,k+1} \geq Q_{k,k+1}, \label{PQassu1}\\
        &P_{k+1,k} \leq Q_{k+1,k}.\label{PQassu2}
\end{align}
 Then,  for every $s$
\[
\lim_{t \rightarrow \infty} \Pr(A_t\geq s) \geq \lim_{t \rightarrow \infty} \Pr(B_t\geq s).
\]
and
\[
\lim_{t \rightarrow \infty} \Pr(A_t\leq s) \leq \lim_{t \rightarrow \infty} \Pr(B_t\leq s).
\]
In other words, the stationary distribution of $A_t$ stochastically dominates the stationary distribution of $B_t$.
\end{lemma}
\begin{proof}
    Let $\alpha_0=\beta_0=1$ and 
$$\alpha_k=\frac{P_{k,k+1}}{P_{k+1,k}},\qquad k\in\{1,2,\cdots, n-1\}.$$
$$\beta_k=\frac{Q_{k,k+1}}{Q_{k+1,k}},\qquad k\in\{1,2,\cdots, n-1\}.$$
The assumptions in \eqref{PQassu1} and \eqref{PQassu2} imply that $\alpha_k\geq \beta_k$ for all $k$.
 Since the birth-death Markov chain is aperiodic and irreducible, the stationary distribution is unique and given by \cite{Gallager1996}
 \begin{align*}
\lim_{t \rightarrow \infty} \Pr(A_t=i) &=\frac{\alpha_0\alpha_1\alpha_2\cdots\alpha_{i-1}}{\sum_{j=1}^n\alpha_0\alpha_1\alpha_2\cdots\alpha_{j-1}}, \qquad i=1,\cdots, n,
\\
\lim_{t \rightarrow \infty} \Pr(B_t=i) &=\frac{\beta_0\beta_1\beta_2\cdots\beta_{i-1}}{\sum_{j=1}^n\beta_0\beta_1\beta_2\cdots\beta_{j-1}}, \qquad i=1,\cdots, n.
 \end{align*}
 Let
\begin{align*}g(x_0,x_1,\cdots, x_{n-1})=
\frac{\sum_{i=s}^nx_0x_1x_2\cdots x_{i-1}}{\sum_{j=1}^nx_0x_1x_2\cdots x_{j-1}}.
 \end{align*}
Then,
 \begin{align*}
\lim_{t \rightarrow \infty} \Pr(A_t\geq s) &=g(\alpha_0,\alpha_1,\cdots, \alpha_{n-1})\\
\lim_{t \rightarrow \infty} \Pr(B_t\geq s) &=g(\beta_0,\beta_1,\cdots, \beta_{n-1})
 \end{align*}
 Since $g(x_0,x_1,\cdots, x_{n-1})$ is increasing in $x_k$ for every $k$ and using the fact that $\alpha_k\geq \beta_k$ for all $k$, we get the desired result. 
 
\end{proof}

\end{document}